\documentclass[runningheads]{llncs}
\usepackage{graphicx}
\usepackage{tikz}
\usepackage{comment}
\usepackage{amsmath,amssymb} 
\usepackage[accsupp]{axessibility}  
\usepackage{enumitem}
\usepackage{url}
\usepackage{lipsum} 
\usepackage{algorithmicx}
\usepackage[ruled]{algorithm}
\usepackage{algpseudocode}
\usepackage[switch]{lineno}
\usepackage{color}
\usepackage{multirow}
\usepackage[normalem]{ulem}
\usepackage{microtype}
\usepackage{epsfig}
\usepackage{pifont}
\usepackage{dsfont}
\usepackage{textcomp}
\usepackage{wrapfig}
\usepackage{overpic}
\usepackage[export]{adjustbox}
\usepackage{pgfplots}
\usepackage{ulem}
\usepackage{subcaption}
\usepackage{relsize}
\usepackage{caption}
\DeclareCaptionFont{9pt}{\fontsize{9pt}{10pt}\selectfont}
\let\oldcenter\center
\let\oldendcenter\endcenter
\renewenvironment{center}{\setlength\topsep{0pt}\oldcenter}{\oldendcenter}
\usepackage{hyperref}
\usepackage{cleveref}

\usepackage{xcolor}

\newcommand{\VG}[1]{{\color{violet}#1}}

\newcommand{\net}{\mathcal{G}}
\newcommand{\loss}{\mathcal{L}}
\newcommand{\netTwo}{\mathcal{G}_1}
\newcommand{\netThree}{\mathcal{G}_2}
\newcommand{\R}{\mathbb{R}}
\newcommand{\inputSpace}{\mathcal{X}}
\newcommand{\outputSpace}{\mathcal{Y}}
\newcommand{\invarianceSet}{S}
\newcommand{\orbitMapping}{h}
\DeclareMathOperator{\argmax}{\arg \max}

\newtheorem{fact}{Fact}
\crefname{fact}{Fact}{Facts}
\crefname{proposition}{Proposition}{Propositions}
\crefname{example}{Example}{Examples}

\pgfplotsset{compat=1.17}

\begin{document}
\pagestyle{headings}
\mainmatter

\def\ACCV22SubNumber{677}  

\title{A Simple Strategy to Provable Invariance via Orbit Mapping} 
\titlerunning{A Simple Strategy to Provable Invariance}
\authorrunning{Gandikota et al.}

\author{Kanchana Vaishnavi Gandikota\textsuperscript{1}, Jonas Geiping\textsuperscript{2}, Zorah Lähner\textsuperscript{1}, Adam Czapli\'nski\textsuperscript{1}, Michael Möller\textsuperscript{1}}
\institute{\textsuperscript{1} University of Siegen,\textsuperscript{2} University of Maryland}
\maketitle
\begin{abstract}
Many applications require robustness, or ideally invariance, of neural 
networks to certain transformations of input data. Most commonly, this requirement is addressed by training data augmentation, using adversarial training, or defining network architectures that include  the desired invariance by design. 
In this work, we propose a method to make network architectures  provably invariant with respect to group actions by choosing one element from a (possibly continuous) orbit based on a fixed criterion.   In a nutshell, we intend to 'undo' any possible transformation before feeding the data into the actual network. Further, we empirically analyze the properties of different approaches which incorporate invariance via training or architecture, and demonstrate the advantages of our method in terms of robustness and computational efficiency. 
In particular, we investigate the robustness with respect to rotations of images (which can hold up to discretization artifacts) as well as the provable orientation and scaling invariance of 3D point cloud classification. 
\end{abstract}
\section{Introduction}\label{sec:introduction}
Deep neural networks have revolutionized the field of computer vision over the past decade. Yet,  deep networks trained in a straight-forward way often lack desired robustness. In image classification, for instance, rotational, scale, and shift invariance are often highly desirable properties. While training deep networks with millions of realistic images in datasets like Imagenet \cite{imagenet} confers some degree of in/equi-variance  \cite{tensmeyer2016improving,olah2020naturally,lenc2018understanding}, these properties however, cannot be guaranteed. On the contrary, networks are susceptible to adversarial attacks with respect to these transformations (see e.g. \cite{engstrom2017rotation,finlayson2019adversarial,zhao2020isometry,9665906}), and small perturbations can significantly affect their predictions. To counteract this behavior, the two major directions of research are to either modify the training procedure or the network architecture.
\begin{figure*}
\centering
\begin{subfigure}[b]{.22\textwidth}
\centering
\includegraphics[width=.45\textwidth]{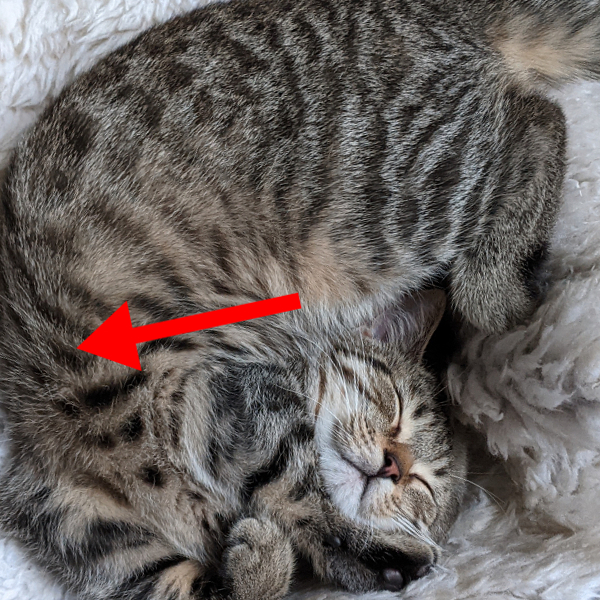}\hspace{0.05cm}
\includegraphics[width=.45\textwidth]{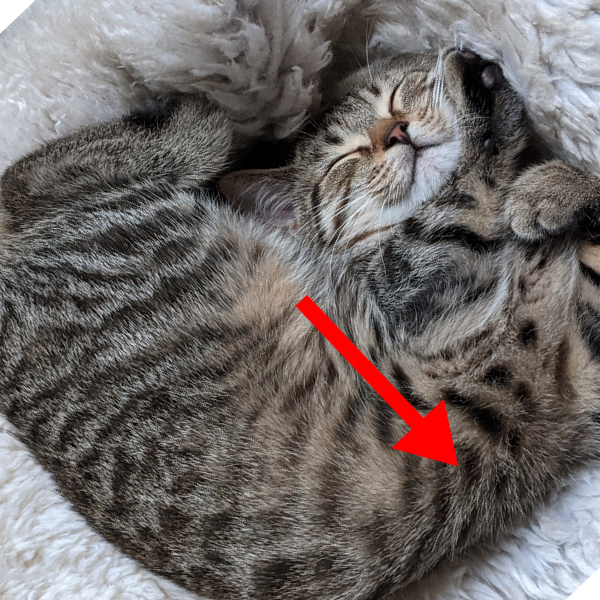}\vspace{0.05cm}\\
\includegraphics[width=.45\textwidth]{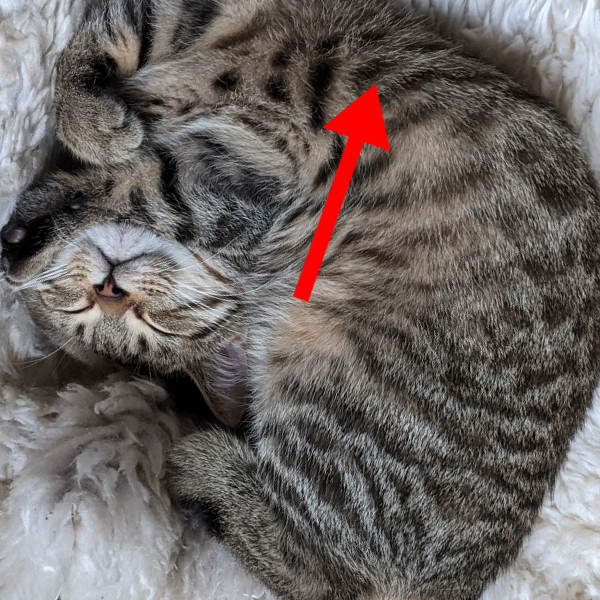}\hspace{0.05cm}
\includegraphics[width=.45\textwidth]{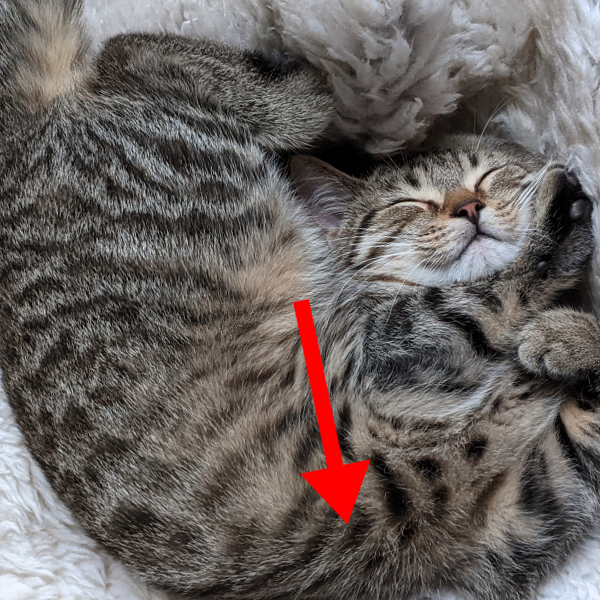}\quad
\tiny{a)~Samples of the orbit}
\end{subfigure}
\begin{subfigure}[b]{.19\textwidth}
\centering
\includegraphics[width=\textwidth]{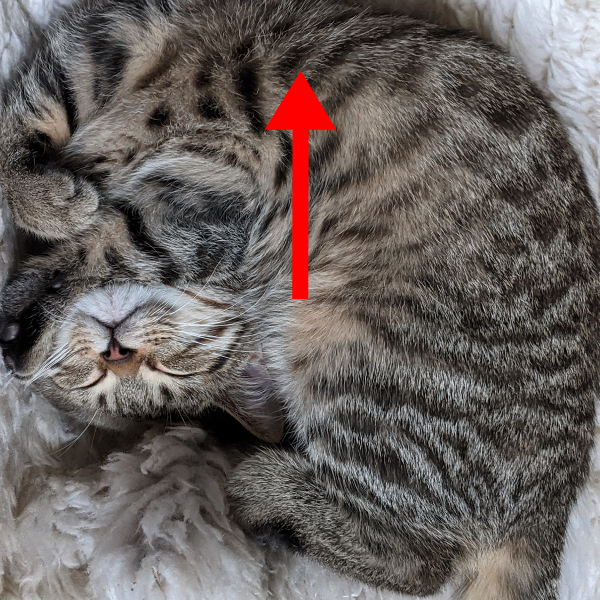}\quad
\tiny{b) Orbit mapping element}
\end{subfigure}
\begin{subfigure}[c]{.35\textwidth}
\centering
\includegraphics[width=\textwidth,margin={0 0 0 -4cm}]{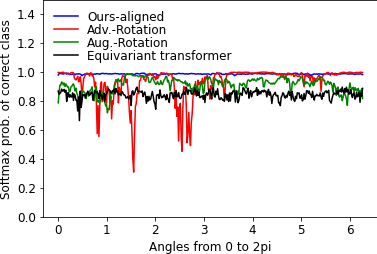}
\end{subfigure}
    \caption{(Left) Picture of a cat in 4 different rotation samples from the continuous orbit of rotations. Our orbit mapping selects the element with mean gradient direction  (marked in red)  along circle pointing upwards.
    (Right) Softmax probabilities of the true label when rotating an image by $0^\circ-360^\circ$. Our method (in blue) is \emph{robust for any angle}, which cannot be guaranteed through data augmentations (green) or adv. training (red).}
    \label{fig:teaser}
\end{figure*}
Modifications of the training procedure replace the common training of a network $\net$ with parameters $\theta$ on training examples $(x^i,y^i)$ via a loss function $\loss$,
\begin{align}
    \min_\theta \sum_{\text{examples i}} \loss(\net(x^i;\theta);y^i),
\end{align}
with a loss function that considers all perturbations in a given set $\invarianceSet$ of transformations to be invariant towards. The most common choices are taking the mean loss of all predictions $\{\net(g(x^i);\theta) ~|~ g \in \invarianceSet\}$ (training with \textit{data augmentation}), or the maximum loss among all predictions (\textit{adversarial training}). However, such training schemes cannot guarantee provable invariance. In particular, training with data augmentation is far from being robust to transformations as illustrated in  Fig.~\ref{fig:teaser}.
The plot shows the softmax probabilities of the true label when feeding the exemplary image at rotations ranging from 0 to $2\pi$ into a network trained with rotational augmentation (green), adversarial training (red) and undoing rotations using a learned network (black). As we can see, rotational data augmentation is not  sufficient to truly make a classification network robust towards rotations, and even the significantly more expensive adversarial training shows instabilities. 

While modifications of the training scheme remain the best option for complex or hard-to-characterize transformations, more structured transformations, e.g., those arising from a group action, allow  modifications to the network architecture to yield provable invariance. As opposed to previous works that largely rely on the ability to enlist all transformations of an input $x$ (i.e., assume a finite \textit{orbit}), we propose to make  neural networks invariant by selecting a specific element from a (possibly infinite) orbit generated by a group action, through an application-specific \textit{orbit mapping}. Simply put, we  undo and fix the transformation or pose. 
Our proposed approach is significantly easier to train than adversarial training methods while being at least equally performant, robust, and computationally cheaper.
We illustrate these findings on the rotation invariant classification of images (on which discretization artifacts from the interpolation after any rotation play a crucial role) as well as on the scale, rotation, and translation invariant classification of 3D point clouds. Our contributions can be summarized as follows:
\begin{itemize}[topsep=0pt]
    \item We present \textit{orbit mapping}, a simple way to adapt neural networks to be in-(or equi)variant to transformations from sets $\invarianceSet$ associated with a group action.
    \item We propose a gradient based orbit mapping strategy for image rotations, which can provably select unique orientation for continuous image models.
    \item Our proposed orbit mapping 
improves robustness of standard networks to transformations even \emph{without} additional changes in training or architecture.
\item Existing invariant approaches also demonstrate gain in robustness to discrete image rotations when combined with orbit mapping.
\item We demonstrate orbit mappings to provable scale and orientation invariant 3D point cloud classification using well known scale normalization and PCA. 
\end{itemize}
\section{Related Work}\label{sec:related}
Several approaches have been developed in the literature to encourage models to exhibit invariance or robustness to desired transformations of data. These include i)~data augmentation using desired transformations, ii)~regularization to encourage  network output to be robust to transformations on the input \cite{simard1991tangent}, iii)~adversarial training~\cite{engstrom2019exploring,Wang_2022_CVPR} and regularization~\cite{yang2019invariance}, iv)~unsupervised or self-supervised pretraining to learn transformation robust representations~\cite{ANSELMI2016112,noroozi2016unsupervised,komodakis2018unsupervised,zhang2019aet,pmlr-v161-gu21a}, v)~parameterized learning of augmentations to learn invariances from training data\cite{wilk2018learning,benton2020learning}, vi)~use of hand-crafted invariant shallow  \cite{sheng1994orthogonal,yap2010polar,tan1998rotation,lazebnik2005sparse,MANTHALKAR20032455} or deep ~\cite{6522407,Sifre_2013_CVPR,Oyallon_2015_CVPR} features for downstream classification tasks vii)~incorporating desired invariance properties in to the  network design~\cite{cohen2016group,Worrall_2017_CVPR,Weiler2019GeneralES,zhang2020learning,yu2020deep}, and viii)~train time/test time data transformation. Recent works \cite{balunovic2019certifying,Fischer2020certified} have also explored certifying geometric robustness of  networks. The approaches i)-v) can  improve robustness but cannot yield provable invariance to transformations. Hand-crafting features can yield desired invariance, but is difficult and often sacrifices accuracy. Provable invariance to a finite number of transformations is achievable by applying all such transformations to the each input data point and pooling the corresponding features  \cite{manay2006integral,laptev2016ti}. While this strategy can even be applied only during test time, it can not be extended to sets with infinitely many transformations. 
Recent approaches \cite{cohen2016group,ravanbakhsh2017equivariance,Weiler2019GeneralES} incorporate in-/equivariances when the desired transformations of the data can be formulated as a group action, e.g. enforcing equivariance  in each layer separately. 
Layer wise approaches for equivariance to finite groups such as \cite{cohen2016group} typically use all possible transformations at each layer.\\
\textbf{Canonicalization}  
Closely related to our approach are methods which align input to a normalized or canonical pose.  The use of PCA or scale renormalization are well known approaches to normalizing point clouds. However, PCA-based pose canonicalization is known to suffer from  ambiguities, and  learning based approaches  \cite{xiao2020endowing,yu2020deep,Li_2021_ICCV} have been proposed for disambiguation. Several recent works directly leverage deep learning for 3d pose canonicalization, for example training with ground truth poses \cite{rempe2020caspr,wang2019normalized} or  self-supervised learning \cite{sun2021canonical,spezialetti2020learning,sajnani2022condor}. For 2D images, PCA-based canonicalization is possible only with binary images \cite{rehman2018automatic}; the use of Radon transformations \cite{1424459} requires an expensive, fine discretization of continuous rotations. The use of spatial transformer networks \cite{jaderberg2015transformer} is an alternate learning based approach to 2D/3D pose normalization which can be used along with an application-dependent coordinate transformation \cite{tai2019equivariant,esteves2018polar}. Such learning-based approaches, however, require additional training with data augmentation and cannot guarantee invariance. Since our orbit mappings essentially select a canonical group orbit element, our work can be interpreted as a formalization of canonicalization for group transformations. In contrast to learning based approaches, we select a canonical element from the orbit using simple analytical solutions, which can improve robustness even without data augmentations.\\
\textbf{Provable Rotational In-/equivariance in 2D}
Several works \cite{Sifre_2013_CVPR,Oyallon_2015_CVPR,cohen2016group,marcos2017rotation,veeling2018rotation,marcos2016learning} have considered layer wise equivariance to discrete rotations using multiple rotated versions of filters at each layer, which was formalized using group convolutions in \cite{cohen2016group}.
While \cite{cohen2016group,marcos2017rotation,veeling2018rotation,marcos2016learning} learn these filters by training, \cite{Sifre_2013_CVPR,Oyallon_2015_CVPR} make use of rotated and scaled copies of fixed wavelet filters at each layer.
For equivariance to continuous rotations, Worrall et al. \cite{Worrall_2017_CVPR} utilize circular harmonic filters at each layer. All these layer wise approaches for group equivariance in images were unified in a single framework in \cite{Weiler2019GeneralES}. Instead of layer-wise approaches, \cite{fasel2006rotation,laptev2016ti,henriques2017warped}  pool the features of  multiple rotated copies of images input to the network. \\
\textbf{Rotation Invariance in 3D}
Due to the different representations of 3D data (e.g. voxels, point clouds, meshes), many strategies exist. Some techniques  for image invariances can be adapted to voxel representations, e.g. probing several rotations at test time \cite{wu2015shapenets,Wang-2017-ocnn},  use of rotationally equivariant convolution kernels \cite{weiler20183d,thomas2018tensor,fuchs2020se}.    Spatial transformers have also been used to learn 3D pose normalization,  e.g.  in the classical PointNet architecture \cite{qi2017pointnet}, and its extension PointNet++ \cite{qi2017pointnetpp} 
which additionally considers hierarchical and neighborhood information. While point clouds do not suffer from discretization artifacts after rotations, they struggle with less clear neighborhood information due to unordered coordinate lists. 
\cite{ZhangR_18_gcnn_point_cloud} solve this by adding hierarchical graph connections to point clouds and using graph convolutions. 
However, the features learned using graph convolutions still depend on the rotation of the input data. \cite{horie2020isometric,pmlr-v139-satorras21a} propose graph convolution networks equivariant to isometric transformations.
\cite{esteves2018learning,rao2019spherical} project point clouds onto 2D sphere and employ spherical convolutions to achieve rotational equivariance. \cite{deng2018ppf} and \cite{zhao20193d} achieve rotation invariance on point clouds by considering pairs of features in the tangent plane of each point. While local operations and convolutions on the surface of triangular meshes are invariant to global rotations by definition \cite{monti2016monet}, they however do not capture global information. 
MeshCNN \cite{hanocka2019meshcnn} addresses this by adding pooling operations through edge collapse. \cite{sharp2020diffusion} defines a representation independent network structure based on heat diffusion which can balance between local and global information.
\section{Proposed Approach}
\label{sec:proposedApproach}
Our idea is straightforward. We make neural networks invariant by consistently selecting a fixed element from the orbit of group transformations, i.e, we modify the input pose  such that every element from the orbit of transformations  maps to the same canonical element. For example, different rotated versions of an image are mapped to have the same orientation as visualized in Fig.~\ref{fig:rotationIllustration}. In conjunction with such \textit{orbit mapping}, any standard network architecture can achieve provable invariance. In the following, we formalize our approach to achieve invariance.
\subsection{Invariant Networks w.r.t. Group Actions}
We consider a network $\net$ to be a function $\net: \inputSpace \times \R^p \rightarrow \outputSpace$ that maps data $x \in \inputSpace$ from some suitable input space $\inputSpace$ to some prediction $\net(x;\theta) \in \outputSpace$ in an output space $\outputSpace$ where the way this mapping is performed depends on parameters $\theta \in \R^p$. The question is how, for a given set $\invarianceSet \subset \{g: \inputSpace \rightarrow \inputSpace\}$ of transformations of the input data, we can achieve the \textit{invariance} of $\net$ to $\invarianceSet$ defined as
\begin{align}
\label{eq:invariance}
    \net(g(x);\theta) = \net(x;\theta) \quad \forall x \in \inputSpace, ~ g \in \invarianceSet, ~ \theta \in \R^p.
\end{align}
The invariance of a network with respect to transformations in $\invarianceSet$ is of particular interest when $\invarianceSet$ induces a \textit{group action}\footnote{\smaller{A (left) group action of  a group $S$ with the identity element $e$,  on a set $X$ is a map 
$\sigma : S\times X\rightarrow X,$
that satisfies i)~$\sigma(e,x)=x$ and  ii)~$\sigma(g,\sigma(h,x))=\sigma(gh,x)$,  $\forall g,h\in S$ and  $\forall x\in X$.  When the action being considered is clear from the context, we write  $g(x)$ instead of $\sigma(g,x)$.
}} on $\inputSpace$, which is what we will assume about $\invarianceSet$ for the remainder of this paper. 
Of particular importance for the construction of invariant  networks, is the set of all possible transformations of input data $x$,
\begin{equation}
\label{eq:orbit}
\invarianceSet \cdot x = \{g(x) ~|~ g \in \invarianceSet\},
\end{equation}
which is called the \textit{orbit of $x$}. A basic observation for constructing invariant networks is that any network acting on the orbit of the input is automatically invariant to transformations in $\invarianceSet$: 
\begin{fact}\textbf{Characterization of Invariant Functions via the Orbit:}
\label{fact:invariance}
Let $\invarianceSet$ define a group action on $\inputSpace$. A network $\net: \inputSpace \times \R^p \rightarrow \outputSpace$ is invariant under the group action of $\invarianceSet$ if and only if it can be written as $\net(x;\theta) =  \netTwo(\invarianceSet \cdot x;\theta)$ for some other 
network $\netTwo: 2^\inputSpace \times \R^p \rightarrow \outputSpace$. 
\end{fact}
The above observation is based on the fact that $\invarianceSet \cdot x = \invarianceSet \cdot g(x)$ holds for any $g \in \invarianceSet$, provided that $\invarianceSet$ is a group. Although not taking the general perspective of \cref{fact:invariance}, approaches, like \cite{laptev2016ti}, which integrate (or sum over finite elements of) the mappings of $\net$ over a  (discrete) group can be interpreted as instances of \cref{fact:invariance} where  $\netTwo$ corresponds to the  summation.  
Similar strategies of applying all transformations in $\invarianceSet$ to the input $x$ can be pursued for the design of equivariant networks, see Appendix  A.
\subsection{Orbit Mappings}
While \cref{fact:invariance} is stated for general (even infinite) groups, realizations of such constructions from the literature often assume a finite
orbit. In this work we would like to include an efficient solution even for cases in which the orbit is not finite, and utilize \cref{fact:invariance} in the most straight-forward way: We propose to construct provably invariant networks  $\net(x;\theta) =  \netTwo(\invarianceSet \cdot x;\theta)$ by simply using an 
$$\text{\textit{orbit mapping }} \orbitMapping:\{\invarianceSet \cdot x ~|~ x \in \inputSpace \} \rightarrow \inputSpace, $$
which uniquely selects a particular element from an orbit as a first layer in $\netTwo$. Subsequently, we can proceed with any standard network architecture and \cref{fact:invariance} still guarantees the desired invariance. A key in designing instances of orbit mappings is that they should not require enlisting all elements of $\invarianceSet \cdot x$ in order to evaluate $\orbitMapping(\invarianceSet \cdot x)$. 
Let us provide more concrete examples of orbit mappings. 
\begin{example}[Mean-subtraction]
A common approach in data classification tasks is to first normalize the input by subtracting its mean. Considering $\inputSpace = \R^n$ and $\invarianceSet = \{g:\R^n \rightarrow \R^n~|~ g(x) = x + a\mathds{1}, \text{ for some } a\in \R\} $, with $\mathds{1}\in \R^n$ being a vector of all ones, input-mean-subtraction is an orbit mapping that selects the unique element from any 
$\invarianceSet\cdot x$ which has zero mean. 
\end{example}
\begin{example}[Permutation invariance via sorting]
\label{example:sorting}
Consider $\inputSpace= \R^n$, and $\invarianceSet$ to be all permutations of vectors in $\R^n$, i.e., $\invarianceSet = \{s \in \{0,1\}^{n \times n} ~|~ \sum_i s_{i,j}=1~\forall j,~~ \sum_js_{i,j} = 1~\forall i\}$. We could define 
an orbit mapping that  selects the element from an orbit whose entries are sorted by magnitude in an ascending order. 
\end{example}

With the very natural condition that  orbit mappings really select an element from the orbit, i.e., $\orbitMapping(\invarianceSet \cdot x) \in \invarianceSet \cdot x$, we can readily construct equivariant networks by applying the inverse mapping, see Appendix A. In our \cref{example:sorting}, undoing the sort operation at the end of the network allows to transfer from an invariant, to an equivariant network. 

As a final note, our concept of orbit mappings can further be generalized by $\orbitMapping$ not mapping to the input space $\inputSpace$, but to a different representation, which can be beneficial for particular, complex groups. In geometry processing, for instance, an important group action are isometric deformations of shapes. A common strategy to handle these (c.f. \cite{ovsjanikov12functionalmaps}) is to identify any shape with the eigenfunctions of its Laplace-Beltrami operator  \cite{pinkall93cotan}, which represents a natural (generalized) orbit mapping. We refer to \cite{litany17deepFM,eisenberger2020deepshells,huang2019operatornet} for exemplary deep learning applications. 
\section{Applications}
We will now present two specific instances of orbit mappings for handling continuous rotations of images as well as for invariances in 3D point cloud classification.
\subsection{Invariance to continous image rotations}
\textbf{Images as functions} 
\begin{figure}[t]
    \centering
    \includegraphics[angle=-90,origin=c,clip,width=0.55\linewidth,trim={0 0cm 0cm 0},clip]{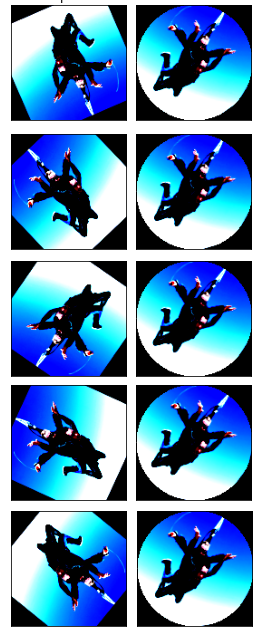}
    \vspace{-6em}
    \caption{Images of different orientations (top) are consistently aligned with the proposed gradient-based orbit mapping (bottom). }
    \label{fig:rotationIllustration}
\end{figure}
Let us consider the important example of invariance to continuous rotations of images. To do so, consider $\inputSpace \subset \{u:\Omega\subset \R^2 \rightarrow \R\}$ to represent images as functions. For the sake of simplicity, we consider grayscale images only, but this extends to color images in a straight-forward way. In our notation $z \in \R^2$ represents spatial coordinates of an image (to avoid an overlap with our previous $x\in \inputSpace$, which we used for the input of a network). We set
\begin{align}    \label{eq:rotationsOnImages}
    \begin{split}
    & \invarianceSet = \{g:\inputSpace \rightarrow \inputSpace ~|~ g\circ u(z) = u(r(\alpha)z),
    \text{ for } \alpha \in \R \},\\
   & \text{and }  r(\alpha) = \begin{pmatrix}\cos(\alpha) & -\sin(\alpha)\\ \sin(\alpha) & \cos(\alpha) \end{pmatrix}.
    \end{split}
\end{align}
As $\invarianceSet$ has infinitely many elements, approaches that worked well for rotations by $90$ degrees like \cite{cohen2016group} are not applicable anymore. 
We instead  propose to uniquely select an element from the continuous orbit of rotation $g\in \invarianceSet$ by choosing a rotation that makes the average gradient of the image  $\int_Z \nabla (g\circ u)(z) ~dz$ over  a suitable set $Z$, e.g. a circle around the image center point upwards. It holds that 
\begin{align*}
 \begin{split}
 \nabla (g\circ u)(z) = r^T(\alpha)\nabla u\left(r(\alpha)z\right) \text{ such that}\\
 \int_Z \nabla(g\circ u)(z) dz = \int_Z r^T(\alpha)\nabla u\left(r(\alpha)z\right) dz.
\end{split}
\end{align*}
Substituting $\varphi=r(\alpha)z$, we obtain
\begin{align}
 \begin{split}
 \int_Z r^T(\alpha)\nabla u\left(r(\alpha)z\right) dz =
 \int_{r^T(\alpha)Z} r^T(\alpha)\nabla u\left(\varphi\right) d\varphi
  =  r^T(\alpha)\int_{Z} \nabla u\left(\varphi\right) d\varphi
 \end{split}
\end{align}
where we used that $Z$ is rotationally invariant.
Thus, choosing a rotation that makes  $\int_Z \nabla (g\circ u)(z) ~dz$  point upwards is equivalent to solving
\begin{align}\label{eq:max_obj}
r(\hat{\alpha}) = \argmax_{r(\alpha)}& \left\langle \begin{pmatrix}1\\0 \end{pmatrix}, r^T(\alpha)\int_Z \nabla u(\varphi) ~d\varphi  \right\rangle
\end{align} 
whose solution is given by  $\hat{\alpha}$ such that
\begin{align}
\begin{pmatrix}
   \cos\hat{\alpha} \\\sin\hat{\alpha}
 \end{pmatrix}= 
    \left( \frac{\int_Z \nabla u(z) ~dz}{\|\int_Z \nabla u(z) ~dz\|}\right).
\label{eq:gradRotation}
\end{align}
Note that \eqref{eq:gradRotation} yields  unique solution to the maximization problem. Since a consistent pose is always selected\footnote{Note that $r^T(\alpha)= r(-\alpha)$, therefore if the predicted rotation for $u(z)$ is $\beta$, then for $u(r(\gamma)z)$, it is $\beta-\gamma$, i.e the same element is consistently selected.}, it is an invariant mapping. When $\int_Z \nabla u(z) ~dz=0$,  any $g \in \invarianceSet$ maximizes \eqref{eq:max_obj}. However, numerically $\int_Z \nabla u(z) ~dz$ rarely evaluates to exact zero and its magnitude of  determines the stability of orbit mapping.

\begin{table*}[t]
\begin{center}
\resizebox{0.7\linewidth}{!}{
\begin{tabular}{l l| l l l| l l l| l l l}
    \hline
    \multirow{2}{*}{Method}&\multirow{2}{*}{OM}(Ours)&\multicolumn{3}{c|}{CIFAR10}& \multicolumn{3}{c|}{HAM10000}&\multicolumn{3}{c}{CUB200}\\
    &&Clean&Avg.&Worst&Clean&Avg.&Worst&Clean&Avg.&Worst\\
    \hline
    \multirow{2}{*}{Std.}  &  \ding{55} & \textbf{93.98} &  40.06&1.31 &93.82&91.73&82.52&\textbf{77.41}&53.45&8.07\\
&\ding{51} Train+Test& 87.99& 84.12&68.60&93.31&91.38&87.96&71.19&71.56&58.80
\\
\multirow{2}{*}{RA}& \ding{55} &
85.54& 75.99&44.71&93.30&90.81&82.30&69.89&70.12&41.01\\
& \ding{51} Train+Test& 85.40& 81.82&71.09&93.41&92.13&88.55&70.35&70.72&57.54\\
\hline 
STN&\ding{55} &83.74&78.86&54.03&--&--&--&--&--&--\\
ETN&\ding{55} &84.39&80.30&64.08&92.47&90.85&84.32&64.14&66.95&52.85\\
Adv.&\ding{55}&69.32&68.54&50.21&92.28&91.87&85.04&64.54&64.07&42.82\\
Mixed&\ding{55}&91.15&68.37&17.15&93.71&92.13&84.53&68.56&65.91&42.87\\
Adv.-KL&\ding{55}&72.28&70.29&51.05&92.54&91.79&85.42&64.47&64.65&43.04\\
Adv.-ALP&\ding{55}&71.25&70.30&52.29&92.89&91.84&85.98&64.63&64.34&43.63\\
\hline
\multirow{2}{*}{TIpool} &\ding{55}&93.56&66.46&20.22&93.19 & 91.87&88.16&76.80&74.90&59.04\\
&\ding{51} Train+Test& 91.94&\textbf{88.77}&76.26&\textbf{93.83} & 92.05& \textbf{89.81}&76.82&\textbf{77.18}&\textbf{69.19}\\
\multirow{2}{*}{TIpool-RA}& \ding{55}&91.40&84.65&67.28&93.39 & 91.87& 88.12&73.47&74.71&62.82\\
& \ding{51} Train+Test&90.47&87.92&\textbf{80.07}&93.68& \textbf{92.78}& 89.30
&74.78&75.89&67.78 \\
\hline
\end{tabular}}
\end{center}
\caption{Comparison of orbit mapping \emph{(OM)} with  training and architecture based methods. Robustness to rotations is compared using the average and worst case accuracies over 5 runs with test images rotated in steps of $1^\circ$ using bilinear interpolation.}
\label{tab:rob.train_reg.}
\end{table*}

\begin{table*}[t]
\begin{center}
\resizebox{\linewidth}{!}{
\begin{tabular}{l l  l  l l l l l  c l l l }
    \hline
\multirow{2}{*}{Train} &  \multirow{2}{*}{OM}& \multirow{2}{*}{Clean}&  \multicolumn{3}{c}{Average }& \multicolumn{3}{c}{Worst-case}\\
\cline{4-6}\cline{8-10}
&&&Nearest&Bilinear&Bicubic&&Nearest&Bilinear&Bicubic\\
\hline
\multirow{2}{*}{Std.}  &  \ding{55} & \textbf{93.98$\pm$0.32} & 35.12$\pm$0.81& 40.06$\pm$0.44 &42.81$\pm$0.50&&0.79$\pm$0.38 &1.31$\pm$0.13& 2.22$\pm$0.17\\
&\ding{51} Train+Test& 87.99$\pm$0.43&72.40$\pm$0.33& 84.12$\pm$0.55 &86.61$\pm$0.49&&34.57$\pm$0.94 &68.60$\pm$0.81& 74.49$\pm$0.84\\
\hline 
\multirow{3}{*}{RA}& \ding{55} &
85.54$\pm$0.72&80.47$\pm$0.74& 75.99$\pm$0.72 &79.47$\pm$0.65&&45.50$\pm$0.83 &44.71$\pm$0.74& 50.50$\pm$0.78\\
 & \ding{51} Test    &79.26$\pm$0.42
 &	74.93$\pm$0.51& 69.31$\pm$0.65 &73.94$\pm$0.63&&48.93$\pm$0.75 &52.18$\pm$0.91& 58.69$\pm$0.78\\
& \ding{51} Train+Test& 85.40$\pm$0.57&84.37$\pm$0.58& 81.82$\pm$0.59 &84.82$\pm$0.52&&66.22$\pm$0.75 &71.09$\pm$1.01& 76.44$\pm$0.89\\
\hline
\multirow{3}{*}{\shortstack{RA-\\combined}}      &\ding{55}&92.42$\pm$0.21&80.90$\pm$0.64& 82.23$\pm$0.74 &82.71$\pm$0.69&&36.98$\pm$1.27 &48.07$\pm$1.66& 49.51$\pm$1.47\\
&  \ding{51} Test&   82.55$\pm$0.86&76.33$\pm$0.95& 77.93$\pm$0.68 &78.42$\pm$0.64&&45.44$\pm$1.32 &60.23 $\pm$1.24&62.18$\pm$1.33\\
&  \ding{51} Train+Test&86.69$\pm$0.12&84.06$\pm$0.21& 85.27$\pm$0.23 &86.06$\pm$0.20&&61.75$\pm$0.76 &75.29$\pm$0.42& 77.25$\pm$0.27\\
\hline     
Adv.& \ding{55}&69.32$\pm$1.61&61.73$\pm$1.12& 68.54$\pm$0.68 &68.00$\pm$0.31&&36.95$\pm$0.97 &50.21$\pm$0.55& 49.73$\pm$0.98\\
Mixed& \ding{55}&91.15$\pm$0.15&54.55$\pm$0.40& 68.37$\pm$0.66 &68.48$\pm$0.37&&3.86$\pm$0.13 &17.15$\pm$1.25& 16.85$\pm$0.93\\
Adv.-KL& \ding{55}&72.28$\pm$2.05&62.60$\pm$1.72& 70.29$\pm$1.42 &69.84$\pm$1.29&&32.60$\pm$0.74 &51.05$\pm$2.47& 51.11$\pm$1.03\\
Adv.-ALP&\ding{55}&71.25$\pm$0.97&62.36$\pm$2.19& 70.30$\pm$1.50 &69.71$\pm$1.22&&33.98$\pm$1.44 &52.29$\pm$1.76& 52.57$\pm$1.57\\
STN&\ding{55}&83.74$\pm$0.50&81.94$\pm$0.51& 78.86$\pm$0.73 &82.21$\pm$0.55&&51.23$\pm$1.01 &54.03$\pm$1.36& 59.65$\pm$1.31\\
ETN&\ding{55}&84.39$\pm$0.09&82.98$\pm$0.28& 80.30$\pm$0.55 &83.31$\pm$0.31&&59.40$\pm$0.76 &64.08$\pm$0.78& 68.75$\pm$0.83\\
Augerino&\ding{55}&83.68$\pm$0.76&80.17$\pm$0.70& 82.27$\pm$0.69 &81.69$\pm$0.72&&52.44$\pm$0.66 &60.36$\pm$1.00& 60.63$\pm$0.94\\

TIpool&\ding{55}&93.56$\pm$0.25&55.96$\pm$0.39& 66.46$\pm$1.36 &70.70$\pm$0.77&&3.14$\pm$1.09 &20.22$\pm$1.51& 27.88$\pm$1.09\\
TIpool-RA&\ding{55}&91.40$\pm$0.17&87.50$\pm$0.24& 84.65$\pm$0.51 &87.31$\pm$0.29&&66.52$\pm$1.31&67.28$\pm$1.03& 72.35$\pm$0.83\\
TIpool&\ding{51}Train+Test&91.94$\pm$0.38&78.66$\pm$0.83& 88.77$\pm$0.51 &\textbf{90.76$\pm$0.40}&&42.01$\pm$1.07 &76.26$\pm$1.12& 81.46$\pm$1.02\\
TIpool-RA&\ding{51}Train+Test&90.47$\pm$0.36&\textbf{89.37$\pm$0.36}& 87.92$\pm$0.36 &89.91$\pm$0.34&&\textbf{74.51$\pm$0.79} &80.07$\pm$0.69& 83.76$\pm$0.60\\
TIpool-RA&\multirow{2}{*}{\ding{51}Train+Test}&\multirow{2}{*}{91.09$\pm$0.40}&\multirow{2}{*}{89.02$\pm$0.30}&\multirow{2}{*}{\textbf{90.13$\pm$0.34}} & \multirow{2}{*}{90.64$\pm$0.30} &&\multirow{2}{*}{70.18$\pm$1.12} &\multirow{2}{*}{\textbf{82.71$\pm$0.62}}& \multirow{2}{*}{\textbf{84.26$\pm$0.41}}\\
combined\\
\hline
\end{tabular}
}
\end{center}
\caption{Effect of augmentation on robustness to rotations with different interpolations. Shown are clean accuracy on standard CIFAR10 test set,  average and worst-case accuracies on rotated test set with mean and standard deviations over 5 runs. }
\label{tab:train_test}
\end{table*}
\hspace{-1.5em}\textbf{Discretization~}
For a discrete (grayscale) image given a matrix $\tilde{u} \in \R^{n_y \times n_x}$, we first apply Gaussian blur with a standard deviation of $\sigma =1.5 $ (to  reduce the effect of noise and create a smooth image), and subsequently construct an underlying continuous function $u:\Omega\subset \R^2 \rightarrow \R$ by bilinear interpolation. For the set $Z$ we choose two circles of radii $0.05$ and $0.4$ (for $\Omega$ being normalized to $[0,1]^2$). We approximate the integral by a  sum over finite evaluations of the derivative along each circle, using exact differentiation of the continuous image model. 
This strategy can stabilize arbitrary rotations  successfully as illustrated in Fig.~\ref{fig:rotationIllustration}. However, in practice, the magnitude of $\int_Z \nabla u(z) ~dz$ and interpolation artifacts affect the stability of the orbit mapping. We analyze the stability of the proposed gradient based orbit-mapping for discrete images in Appendix C, where we observe that use of forward or central differences to approximate gradients further deteriorates the stability of orbit mapping. Since the orbit mapping for discrete images has instabilities, exact invariance to rotations cannot be guaranteed. Even when the integral values are large leading to a stable orbit mapping,  our approach does not need to give the same rotation angle for semantically similar content, for example, different cars are not necessarily rotated to have the same orientation. Due to these reasons, our approach can further benefit from augmentation.\vspace{0.25em}\\
\textbf{Experiments} 
To evaluate our approach, we use orbit mapping in conjunction with  image classification networks on three datasets: On CIFAR10, we train a Resnet-18 \cite{he2016deep} from scratch. On the HAM10000 skin image dataset \cite{tschandl2018ham10000}, we finetune an NFNet-F0 network~\cite{brock2021high}, and on CUB-200~\cite{wah2011caltech} we finetune a Resnet-50~\cite{he2016deep}, both  pretrained on ImageNet. While the datasets CIFAR10 and CUB-200 have an inherent variance in orientation, for the HAM10000 skin lesion classification,  exact rotation invariance is desirable. Finally, we also perform experiments with RotMNIST using state of the art E2CNN network\cite{Weiler2019GeneralES}. The details of the protocol used for training all our networks as well as some additional experiments are provided in the Appendix E.  We compare with following approaches on CIFAR10, HAM10000, and CUB-200: 
\emph{i)~adversarial training:}
$\min_\theta \sum_{\text{examples i}} \loss(\net(\hat{x}^i;\theta);y^i)$, for $\hat{x}^i=\argmax_{z \in S\cdot x^i}\loss(\net(z);y^i)$. 
This is approximated by selecting the worst out of  10 different random rotations for each image  in  every iteration, following \cite{engstrom2019exploring}. It is referred to as Adv. in  Tab.~\ref{tab:rob.train_reg.}.
\emph{ii)~mixed mode training:}  $\min_\theta \sum_{\text{examples i}} \loss(\net(\hat{x}^i;\theta);y^i)+\loss(\net(x^i;\theta);y^i)$ which uses both natural and adversarial examples $\hat{x}^i$.
\emph{iii)~adversarial training with regularization:} Use of adversarial logit pairing and KL-divergence regularizers~\cite{yang2019invariance}  along with adversarial training (indicated as Adv.-ALP and Adv.-KL in Tab.~\ref{tab:rob.train_reg.}):\\
\indent\emph{a)~adversarial logit pairing~(ALP)}:  $R_{ALP}(\net,x^i,y^i)=\|\net(x^i;\theta)-\net(\hat{x}^i;\theta)\|_2^2$ , \\
\indent\emph{b)~KL-divergence}:$R_{KL}(\net,x^i,y^i)=D_{KL}(\net(x^i;\theta)||\net(\hat{x}^i;\theta))$. \\
\emph{iv)~transformation invariant pooling (TIpool):} which is a provably invariant approach for discrete rotations \cite{laptev2016ti}, where the features of multiple rotated copies of input image are pooled before the final classification. We use 4 rotated copies of images rotated in multiples of 90 degrees. \emph{v)~Spatial transformer networks (STN):} which learns to undo the transformation by training
 using appropriate data augmentation \cite{jaderberg2015transformer}.
\emph{vi)~Equivariant transformer networks (ETN):} which additionally uses appropriate coordinate transformation  along with a learned spatial transformer to undo the transformation \cite{tai2019equivariant}.
We also compare with the simple baseline of augmenting with random rotations, referred to as RA in Tab.~\ref{tab:rob.train_reg.}. Additionally, we also compare with \cite{benton2020learning}, an approach which learns distribution of augmentations  on  the task of rotated CIFAR10 classification, referred to as Augerino in Tab.~\ref{tab:train_test}. We use 4 samples from the learned distribution of augmentations during both training and test. We would also like to point out that adversarial training using the worst of $10$ samples roughly increases the training effort of the underlying model by a factor of~$5$.\vspace{0.25em}\\
\textbf{Results }We measure the accuracy on the original testset(\textit{Clean}), as well as the average (\textit{Avg.}) and  worst-case (\textit{Worst}) accuracies in the orbit of rotations discretized in steps of 1 degree, where `\textit{Worst}' counts an image as misclassified as soon as there exists a rotation at which the network makes a wrong prediction.

As we can see in Tab.~\ref{tab:rob.train_reg.}, networks trained without rotation augmentation perform poorly in terms of both, the average and worst-case accuracy if the data set contains an inherent orientation. While augmenting with rotations during training results in improvements, there is still a huge gap ($\sim 30\%$ for CIFAR10 and CUB200) between the average and worst-case accuracies. While adversarial training approaches ~\cite{engstrom2019exploring,yang2019invariance} improve the performance in the worst case, there is a clear drop in the clean and average accuracies when compared to data augmentation.
 Learned approaches to correct orientation i.e. STN~\cite{jaderberg2015transformer}, ETN~\cite{tai2019equivariant} show an improvement over adversarial training schemes in terms of average and worst case accuracies, when training from scratch, with ETN demonstrating even higher robustness than plain STNs.  
While pooling over features of rotated versions of image provides provable invariance to discrete rotations, this approach is still susceptible to continuous image rotations. The robustness of this approach to continuous rotations is boosted by rotation augmentation, with improvements over even learned transformers. Note that using TI-pooling with 4 rotated copies increases the computation by 4 times.  In contrast, our orbit mapping effortlessly leads to significant improvements in robustness even without augmenting with rotations, with performance  better than  adversarial training, learned transformers and discrete invariance based approaches. Since our orbit mapping for discrete images has some instabilities, our approach also benefits from augmentation with image rotations. Further, when combined with discrete invariant approach~\cite{laptev2016ti}, we obtain the best accuracies for average and worst case rotations.

Even when finetuning networks, we observe that orbit mapping readily improves robustness to rotations over standard training, even without the use of augmentations. Furthermore,  combination of orbit mapping  with the discrete invariant approach of pooling over rotated features yields the best performance.
For the birds dataset with inherent orientation, undoing rotations using ETN significantly improves robustness when compared to adversarial training schemes, which only marginally improve robustness over rotation augmentation. We found it difficult to train  STN with higher accuracies (\emph{Clean/Avg./Worst}) than plain augmentation with rotated images for CUB200 and HAM10000, despite extensive hyperparameter optimization, therefore we do not report the numbers here\footnote{\smaller{We use a single spatial transformer as opposed to multiple STNs used in \cite{jaderberg2015transformer} and train on randomly rotated images.}}.
When the data itself does not contain a prominent orientation as in the HAM10000 data set, the general trend in accuracies still holds (\emph{Clean$>$Avg.$>$Worst}), but the drops in accuracies are not drastic, and adversarial training schemes provide improvements over undoing transormations using ETN. Further, orbit mapping and pooling over rotated images provide comparable improvements in robustness, with their combination achieving the best results.\\
\textbf{Discretization Artifacts:~} It is interesting to see that while 
consistently selecting a single element from the continuous orbit of rotations leads to provable rotational invariance when considering images as continuous functions, discretization artifacts and boundary effects still play a crucial role in practice, and rotations cannot be fully stabilized. As a result, there is still discrepancy between the average and worst case accuracies, and the performance is further improved when our approach also uses rotation augmentation. Motivated by the strong effect the discretization seems to have, we investigate different interpolation schemes used to rotate the image in more detail: Tab.~\ref{tab:train_test} shows the results different training schemes with and without our orbit mapping (\textit{OM}) obtained with a ResNet-18 architecture on CIFAR-10 when using different types of interpolation. Besides standard training  (\textit{Std.}), we use rotation augmentation (\textit{RA}) using the Pytorch-default of nearest-neighbor interpolation, a combined augmentation scheme (\textit{RA-combined}) that applies random rotation only to a fraction of images in a batch using at least one nearest neighbor, one bilinear and one bicubic interpolation. The adversarial training and regularization from \cite{engstrom2019exploring,yang2019invariance}  are trained using bilinear interpolation (following the authors' implementation). 

Results show that interpolation used in image rotation impacts accuracies in all the baselines. Most notably, the worst-case accuracies between different types of interpolation may differ by more than $20\%$, indicating a huge influence of the interpolation scheme. Adversarial training with bi-linear interpolation still leaves a large vulnerability to image rotations with nearest neighbor interpolation. Further, applying an orbit mapping at test time to a network trained with rotated images readily improves its worst case accuracy, however, there is a clear drop in clean and average case accuracies, possibly due to the network not having seen doubly interpolated images during training. While our approach without rotation augmentation is also vulnerable to interpolation effects, it is ameliorated when using orbit mapping along with rotation augmentation. We observe that including different augmentations (RA-combined) improves the robustness significantly. Combining the orbit mapping with the discrete invariant approach~\cite{laptev2016ti}  boosts the robustness, with different augmentations further reducing the gap between clean, average case and worst case performance. 
\begin{table*}[t]
\small
\begin{center}
\resizebox{0.8\linewidth}{!}{
\begin{tabular}{ c c |c c c | c c c}
\hline
 \multirow{2}{*}{Train.} & \multirow{2}{*}{OM} &  \multicolumn{3}{c|}{D4/C4}& \multicolumn{3}{c}{D16/C16}\\
 &&Clean & Avg. & Worst&Clean & Avg. & Worst \\
\hline
Std. & \ding{55} & 98.73$\pm$0.04  & 98.61$\pm$0.04 & 96.84$\pm$0.08&99.16$\pm$0.03  & 99.02$\pm$0.04  & 98.19$\pm$0.08
\\
Std. &  \ding{51}(Train+Test)  &98.86$\pm$0.02  & 98.74$\pm$0.03 & 98.31$\pm$0.05&
99.21$\pm$0.01  & 99.11$\pm$0.03  & 98.82$\pm$0.06
 \\
  \hline
 RA. &\ding{55} &99.19$\pm$0.02  & 99.11$\pm$0.01 & 98.39$\pm$0.05&99.31$\pm$0.02  & 99.27$\pm$0.02  & 98.89$\pm$0.03
\\
 RA. & \ding{51}(Train+Test) & 98.99$\pm$0.03  & 98.90$\pm$0.01 & 98.60$\pm$0.02& 99.28$\pm$0.02  & 99.23$\pm$0.01  & 99.04$\pm$0.02
\\
\hline
\end{tabular}}
\end{center}
\caption{Effect of orbit mapping and rotation augmentation on RotMNIST classification using regular D4/C4 and D16/C16 E2CNN models.  Shown are clean accuracy on standard test set and average and worst-case accuracies on 
test set rotated in steps of 1 degree, with mean and standard deviations over 5 runs.\label{tab:rot_mnit_c4c16}}
\end{table*}

\hspace{-1.5em}\textbf{Experiments with RotMNIST} 
We investigate the effect of orbit mapping on  RotMNIST classification
  with the state of the art network from \cite{Weiler2019GeneralES} employing regular steerable  equivariant models\cite{weiler2018steering}. This model uses 16 rotations and flips of the learned filters (with flips being restricted till layer3).  We also compare with a variation of the same architecture with 4 rotations. We refer to these models as  D16/C16 and D4/C4 respectively. 
 We train and evaluate these models using their publicly available code\footnote{code url \url{https://github.com/QUVA-Lab/e2cnn_experiments}}. Results in Tab.~\ref{tab:rot_mnit_c4c16} indicate that even for these state of the art models, there is a discrepancy between the accuracy on the standard test set and the worst case accuracies, and their robustness can be further improved by orbit mapping.
  Notably, orbit mapping significantly improves worst case accuracy (by around 1.5\%) for D4/C4 steerable model trained without augmenting using rotations, showing gains in robustness even over naively trained D16/C16 model of much higher complexity. Training with augmentation leads to improvement in robustness, with  orbit mapping providing gains further in robustness. However, artifacts due to double interpolation affect  performance of orbit mapping.
\subsection{Invariances in 3D Point Cloud Classification}
\begin{table*}[t]
\small
\begin{center}
\resizebox{0.85\linewidth}{!}{
\begin{tabular}{ c c | l l l| l l l}
\hline
 \multirow{2}{*}{Augment.} & \multirow{2}{*}{Unscaling} &  \multicolumn{3}{c|}{with STN}& \multicolumn{3}{c}{without STN}\\
 \cline{3-5}\cline{6-8}
 &&Clean & Avg. & Worst&Clean & Avg. & Worst \\
\hline
 $[0.8, 1.25]$ & \ding{55} & 86.15$\pm$ 0.52 &24.40$\pm$1.56&0.01$\pm$0.02 &85.31$\pm$0.39&33.57$\pm$2.00& 2.37$\pm$0.06
\\
  $[0.8,1.25]$ &  \ding{51}(Train+Test)  & \textbf{86.15$\pm$ 0.28} &\textbf{86.15$\pm$ 0.28}&\textbf{86.15$\pm$ 0.28}  &85.25$\pm$0.43&85.25$\pm$0.43& 85.25$\pm$0.43
\\
 $[0.8, 1.25]$ & \ding{51}(Test) &  86.15$\pm$ 0.52 & 85.59$\pm$0.79&85.59$\pm$0.79 & 85.31$\pm$0.39&83.76$\pm$0.35& 83.76$\pm$0.35
 \\
  \hline
   $[0.1, 10]$ & \ding{55} & 85.40$\pm$0.46&47.25$\pm$1.36& 0.04$\pm$0.05 &75.34$\pm$0.84&47.58$\pm$1.69& 1.06$\pm$0.87
\\
       $[0.1, 10]$  & \ding{51}(Test)& 85.40$\pm$0.46&85.85$\pm$0.73& 85.85$\pm$0.73 & 75.34$\pm$0.84&81.45$\pm$0.56& 81.45$\pm$0.56\\
     \hline

 $[0.001, 1000]$ &\ding{55} &33.33$\pm$ 7.58&42.38$\pm$ 1.54& 2.25$\pm$0.22 &5.07$\pm$2.37&25.42$\pm$0.73& 2.24$\pm$0.11
\\
  $[0.001, 1000]$  & \ding{51}(Train+Test) & 85.66$\pm$ 0.39&85.66$\pm$ 0.39& 85.66$\pm$ 0.39 & 85.05$\pm$0.43&85.05$\pm$0.43& 85.05$\pm$0.43
\\
\hline\end{tabular}}
\end{center}
\caption{Scaling invariance in 3D pointcloud classification  with PointNet trained on modelnet40, with and without data augmentation, with and without STNs or scale normalization. 
Mean and standard deviations over 10 runs are reported.\label{tab:3d_invariances_scaling}}
\end{table*}
\begin{table*}[t]
\small
\begin{center}
\resizebox{0.75\linewidth}{!}{
\begin{tabular}{ c c c | l | l l | l l }
\hline
\multirow{2}{*}{RA} &  \multirow{2}{*}{ STN}& \multirow{2}{*}{PCA}& \multirow{2}{*}{Clean}& \multicolumn{2}{c}{Rotation}& \multicolumn{2}{c}{Translation}\\
\cline{5-8}
& &  & & Avg. & Worst& Avg. &Worst \\
\hline
 \ding{55} &\ding{51} & \ding{55} & \textbf{86.15$\pm$0.52}&10.37$\pm$0.18&0.09$\pm$0.07&10.96$\pm$1.22&0.00$\pm$0.00\\
 \ding{55} &\ding{55} & \ding{55} &85.31$\pm$0.39&10.59$\pm$0.25& 0.26$\pm$0.10&6.53$\pm$0.12 &0.00$\pm$0.00\\
  \ding{55} &\ding{51} & \ding{51}(Train+Test) &74.12$\pm$ 1.80 & 74.12$\pm$ 1.80&74.12$\pm$ 1.80&74.12$\pm$ 1.80&74.12$\pm$ 1.80\\
  \ding{55} &\ding{55} & \ding{51}(Train+Test) &75.36$\pm$0.70&\textbf{75.36$\pm$0.70}& \textbf{75.36$\pm$0.70}&\textbf{75.36$\pm$0.70 }&\textbf{75.36$\pm$0.70}
\\
 \hline
 \ding{51} &\ding{51} & \ding{55} & 72.13$\pm$ 5.84&72.39$\pm$ 5.60&35.91$\pm$ 4.87&5.35$\pm$0.98&0.00$\pm$0.00\\
 \ding{51} &\ding{55} & \ding{55}& 63.93$\pm$0.65&64.75$\pm$0.57& 45.53$\pm$0.29&3.90$\pm$0.71 &0.00$\pm$0.00\\
  \ding{51} &\ding{51} & \ding{51}(Test) &72.13$\pm$ 5.84&72.96$\pm$ 5.85&72.96$\pm$ 5.85&72.96$\pm$ 5.85&72.96$\pm$ 5.85\\
  \ding{51} &\ding{55} & \ding{51}(Test) & 64.56$\pm$0.91&64.56$\pm$0.91& 64.56$\pm$0.91&64.56$\pm$0.91 &64.56$\pm$0.91\\
  \ding{51} &\ding{51} & \ding{51}(Train+Test) & 72.84$\pm$0.77&  72.84$\pm$0.77 &  72.84$\pm$0.77 &  72.84$\pm$0.77 &  72.84$\pm$0.77\\
    \ding{51} &\ding{55} & \ding{51}(Train+Test) & 74.84$\pm$0.86&74.84$\pm$0.86& 74.84$\pm$0.86&74.84$\pm$0.86 &74.84$\pm$0.86\\
    \hline
\end{tabular}}
\end{center}
\caption{Rotation and translation invariances in 3D pointcloud classification  with PointNet trained on modelnet40, with and without rotation augmentation, with and without STNs or PCA. 
Mean and standard deviations over 10 runs are reported.\label{tab:3d_invariances_rotation}}
\end{table*}
\begin{table*}[t]
\begin{center}
 \resizebox{\linewidth}{!}{ 
\begin{tabular}{c c c|c| c| c | cc | cc | cc |cc}
\multicolumn{3}{c|}{Augmentation} &STN &OM & Clean & \multicolumn{2}{c|}{Scaling} &  \multicolumn{2}{c|}{Rotation} &  \multicolumn{2}{c}{Translation}\\
Scale & RA &Translation& & All & & Avg. & Worst & Avg. & Worst& Avg. & Worst \\
\hline
$[0.8, 1.25]$ & \ding{51} &$[-0.1,0.1]$&\ding{51}&\ding{55}& 72.13$\pm$ 5.84&19.74$\pm$ 4.01&0.16$\pm$ 0.42&72.39$\pm$ 5.60&35.91$\pm$ 4.87&5.35$\pm$0.98&0.00$\pm$0.00\\
$[0.8, 1.25]$ & \ding{51} &$[-0.1,0.1]$&\ding{51}&\ding{51} Test&67.38$\pm$ 7.96&64.88$\pm$ 12.16&64.88$\pm$ 12.16&64.88$\pm$ 12.16&64.88$\pm$ 12.16&64.88$\pm$ 12.16&64.88$\pm$ 12.16\\
$[0.8, 1.25]$ & \ding{51} &$[-0.1,0.1]$&\ding{51}& \ding{51} Train+Test&\textbf{77.52$\pm$1.03}& \textbf{77.52$\pm$1.03}&\textbf{77.52$\pm$1.03}&\textbf{77.52$\pm$1.03}&\textbf{77.52$\pm$1.03}&\textbf{77.52$\pm$1.03}&\textbf{77.52$\pm$1.03}\\
$[0.8, 1.25]$ & \ding{51}&$[-0.1,0.1]$ &\ding{55}& \ding{55}  &63.93$\pm$0.65&12.85$\pm$0.29& 0.27$\pm$0.55&64.75$\pm$0.57& 45.53$\pm$0.29&3.90$\pm$0.71 &0.00$\pm$0.00
\\
$[0.8, 1.25]$ & \ding{51}&$[-0.1,0.1]$ &\ding{55}& \ding{51}Test & 64.71$\pm$0.92&57.10$\pm$1.14& 57.10$\pm$1.14&57.10$\pm$1.14& 57.10$\pm$1.14&57.10$\pm$1.14 &57.10$\pm$1.14\\
$[0.8, 1.25]$ & \ding{51}&$[-0.1,0.1]$ &\ding{55}& \ding{51}Train+Test & 74.41$\pm$0.58&74.41$\pm$0.58& 74.41$\pm$0.58&74.41$\pm$0.58& 74.41$\pm$0.58&74.41$\pm$0.58 &74.41$\pm$0.58\\
\hline
\end{tabular} }
\end{center}
\caption{Combined Scale, rotation and translation invariances in 3D pointcloud classification  with PointNet trained on modelnet40, with data augmentation and analytical inclusion of each invariance. 
Mean and standard deviations over 10 runs are reported. \label{tab:3d_invariances_sim}}
\end{table*}
Invariance to orientation and scale is often desired in networks classifying objects given as $3$D point clouds. 
Popular architectures, such as PointNet \cite{qi2017pointnet} and its extensions \cite{qi2017pointnetpp}, rely on the ability of spatial transformer networks to learn such invariances  by training on large datasets and extensive data augmentations. 
We analyze the robustness of these networks to transformations  with experiments using Pointnet on \textit{modelnet40} dataset \cite{wu2015shapenets}. We compare the class accuracy of the final iterate for the clean validation set \textit{(Clean)}, and transformed validation sets  in the average \textit{(Avg.)} and worst-case \textit{(Worst)}.
We show that PointNet performs better with 
our orbit mappings than with augmentation alone.

In this setting, $\inputSpace = \R^{d \times N}$ are $N$ many $d$-dimensional coordinates (usually with $d=3$). The desired group actions for invariance are left-multiplication with a 
rotation matrix, and multiplication with any number $c\in\R^+$ to account for different scaling. We also consider translation by adding a fixed coordinate $c_t \in \mathbb{R}^3$ to each entry in $\inputSpace$. 
Desired invariances in point cloud classification range from class-dependent variances to geometric properties. 
For example, the classification of airplanes should be invariant to the specific wing shape, as well as the scale or translation of the model. 
While networks can learn some invariance from training data, our experiments show that even simple transformations like scaling and translation are not learned robustly outside the scope of what was provided in the training data, see Tabs.~\ref{tab:3d_invariances_scaling}, \ref{tab:3d_invariances_rotation}, \ref{tab:3d_invariances_sim}. 
This is surprising, considering that both can be undone by centering around the origin and re-scaling.\vspace{0.2em}\\
\textbf{Scaling}
Invariance to scaling can be achieved in the sense of Sec.~\ref{sec:proposedApproach} by scaling input point-clouds by the average distance of all points to the origin. Our experiments show that this leads to robustness against much more extreme transformation values without the need for expensive training, both for average as well as worst-case accuracy.
We tested the worst-case accuracy on the following scales: $\lbrace 0.001, 0.01, 0.1, 0.5, 1.0, 5.0, 10, 100, 1000 \rbrace$. 
While our approach performs well on all cases, training PointNet on random data augmentation in the range of possible values actually reduces the accuracy on clean, not scaled test data. This indicates that the added complexity of the task cannot be well represented within the network although it includes  spatial transformers. 
Even when restricting the training to a subset of the interval of scales, the spatial transformers cannot fully learn to undo the scaling, resulting in a significant drop in  average and worst-case robustness, see Tab.~\ref{tab:3d_invariances_scaling}. 
While training the original Pointnet including the desired invariance in the network achieves the best performance, dropping the spatial transformers from the architecture results in only a tiny drop in accuracy with significant gains in training and computation time\footnote{\smaller{Model size of PointNet with STNs is 41.8 MB, and without STNs 9.8 MB}}.
This either indicates that in the absence of rigid deformation the spatial transformers do not add much knowledge and is strictly inferior to modeling invariance, at least on this dataset.\vspace{0.2em}\\
\textbf{Rotation and Translation} In this section, we show that 3D rotations and translations exhibit a similar behavior and can be more robustly treated via 
orbit mapping than through data augmentation. 
This is even more meaningful than scaling as both have three degrees of freedom and sampling their respective spaces requires a lot more examples. 
For rotations, we choose the unique  element of the orbit to be the rotation of $\inputSpace$ that aligns its principle components with the coordinate axes. 
The optimal transformation involves subtracting the center of mass from all coordinates and then applying the singular value decomposition $X=U\Sigma V$ of the point cloud $X$ up to the arbitrary orientation of the principle axes, a process also known as PCA. Rotation and translation can be treated together, as undoing the translation is a substep of PCA.
To remove the sign ambiguity in the principle axes, we choose signs of the first row of $U$ and encode them into a diagonal matrix $D$, such that the final transform is given by $\hat{X} = XV^\top D$. 
We apply this rotational alignment to PointNet with and without spatial transformers and evaluate its robustness to rotations in average-case and worst-case when rotating the validation dataset in $16 \times 16$ increments (i.e. with $16$ discrete angles along each of the two angular degrees of freedom of a 3D rotation).  We test robustness to translations in average-case and worst-case for the following shifts in each of \emph{x, y} and \emph{z} directions: $\lbrace-10.0, -1.0, -0.5, -0.1, 0.1, 0.5, 1.0, 10.0\rbrace$.
Tab.~\ref{tab:3d_invariances_rotation} shows that PointNet trained without augmentation is susceptible in worst-case and average-case rotations and even translations.
The vulnerability to rotations can be ameliorated in the average-case by training with random rotations, but the worst-case accuracy is still significantly lower, even when spatial transformers are employed. Also notable is the high variance in performance of Pointnets with STNs trained using augmentations.
On the other hand, explicitly training and testing with stabilized rotations using PCA does provide effortless invariance to rotations and translations, even without augmentation. Interestingly, the best accuracy here is reached when training PointNet entirely without spatial transformers, which offer no additional benefits when the rotations are stabilized.
The process for invariance against translation is well-known and well-used due to its simplicity and robustness. We show that this approach arises naturally from our framework, and that its extension to rotational invariance inherits the same numerical behavior, i.e., provable invariance outperforms learning to undo the transformation via data augmentation.\vspace{0.2em}\\
\textbf{Combined invariance to Scaling, Rotation, Translation. } Our approach can be extended to make a model simultaneously invariant to scaling, rotations and translations. In this setup, we apply a PCA alignment before normalizing the scale of input point cloud. Tab.~\ref{tab:3d_invariances_sim} shows that PointNet trained with such combined orbit mapping does achieve the desired invariances.
\section{Discussion and Conclusions}
We proposed a simple and general way of incorporating invariances to group actions in neural networks by uniquely selecting a specific element from the orbit of group transformations. This guarantees provable invariance to group transformations for 3D point clouds, and  demonstrates significant improvements in robustness to continuous rotations of images  with a limited computational overhead. However, for images, a large discrepancy between the theoretical provable invariance (in the perspective of images as continuous functions) and the practical discrete setting remains. We conjecture that this is related to discretization artifacts when applying rotations that change the gradient directions, especially at low resolutions.
Notably, such artifacts appear  more frequently in artificial settings, e.g. during data augmentation or when testing for worst-case accuracy, than in photographs of rotating objects that only get discretized once. While we found a consistent advantage of enforcing the desired invariance via orbit mapping rather than training alone,  combination of data augmentation and orbit mappings yields additional advantages (in cases where discretization artifacts prevent a provable invariance of the latter). Moreover, our orbit mapping can be combined with existing invariant approaches for improved robustness.
\appendix
\section{Extension of Orbit Mapping to Equivariant Networks}
The \textit{equivariance} of $\net$  preserves the structure  of transformations $g\in \invarianceSet$ of input data in the elements $y \in \outputSpace$ (including, but not limited to, the case where $\inputSpace \equiv \outputSpace$). The \textit{equivariance} of $\net$ to $\invarianceSet$ is defined as
\begin{align}
\label{eq:equivariance}
    \net(g(x);\theta) = g(\net(x;\theta)) \quad \forall x \in \inputSpace, ~ g \in \invarianceSet, ~ \theta \in \R^p.
\end{align}
We now show that  equivariant networks can be designed  by applying all transformations in $\invarianceSet$ to the input $x$.
\begin{proposition}
\label{fact:equivariance}
Let $\invarianceSet$ define a group action on $\inputSpace$. A network $\net$ is equivariant under the group action of $\invarianceSet$ if it can be written as 
\begin{align}
    \label{eq:fact2}
\net(x;\theta) =  \netTwo(\{g (\netThree (g^{-1}(x);\theta_2)) ~|~ g \in \invarianceSet\};\theta_1)
\end{align}
for some other arbitrary network $\netThree: \inputSpace \times \R^{p_2} \rightarrow \inputSpace$, and a network $\netTwo: 2^\inputSpace \times \R^{p_1} \rightarrow \inputSpace$ that commutes with any element $h \in \invarianceSet$, i.e., for $h \in \invarianceSet$, and $Z\subset \inputSpace$, it satisfies
$\netTwo(h(Z);\theta_2) = h(\netTwo(Z;\theta_2))$, where $h(Z)$ denotes the set obtained by the applying $h$ to every element of $Z$. 
\end{proposition}
\begin{proof}
We want to show that a network satisfying the 
condition (5) is equivariant. Let $h \in S$ be arbitrary. Note that  
\begin{align}
    \{g ~|~ g\in  \invarianceSet\} = 
    \{h^{-1}g ~|~ g \in \invarianceSet\} 
\end{align}
such that a substitution of variables from $g \in \invarianceSet$ to $z = h^{-1}g \in \invarianceSet$ (i.e., $g = hz$ and $z^{-1} = g^{-1}h$) yields
\begin{align*}
    &\{g (\netThree (g^{-1}(h(x));\theta_2)) ~|~ g \in \invarianceSet\} \\
    =& 
    \{h( z (\netThree (z^{-1}(x);\theta_2))) ~|~ z \in \invarianceSet\}.
\end{align*}
This means that we can also write  \begin{align*}
    \net(h(x);\theta) &=  \netTwo(\{h( z (\netThree (z^{-1}(x);\theta_2))) ~|~ z \in \invarianceSet\};\theta_1)\\
    &=  \netTwo(h(\{z (\netThree (z^{-1}(x);\theta_2)) ~|~ z \in \invarianceSet\});\theta_1)\\
    &=  h(\netTwo(\{z (\netThree (z^{-1}(x);\theta_2)) ~|~ z \in \invarianceSet\});\theta_1)\\
    &= h(\net(x;\theta))
\end{align*}
which yields the desired equivariance under the assumed commutative property.
\end{proof}
 The work \cite{cohen2016group} can be interpreted as an instance of the construction in \cref{fact:equivariance}, where equivariant linear layers w.r.t. rotations by 90 degrees are obtained by choosing $\netThree$ to be a simple convolution and $\netTwo$ to be the summation over all (finitely many) elements of the set. Subsequently, they nest these layers with component-wise (and therefore inherently equivariant) non-linearities.
 
 While Proposition 1 is stated for general groups, realizations of such constructions often rely on the ability to list an entire orbit of the group. In the following we show an efficient solution to obtain equivariant networks using orbit mapping.
 \begin{proposition}[Orbit mapping for equivariant networks]
\label{fact:argmaxInvariance}
Let $\orbitMapping$ be an orbit mapping that satisfies  $\orbitMapping(\invarianceSet \cdot x) \in \invarianceSet \cdot x$ for all $x$. Any network $\net: \inputSpace \times \R^p \rightarrow \inputSpace$ that can be written as
\begin{align}
    \label{eq:equivariantArgmaxNets}
    \net(x;\theta) = \hat{g}^{-1}(\netThree(\hat{g}(x);\theta))
\end{align}
for an arbitrary network $\netThree: \inputSpace \times \R^p \rightarrow \inputSpace$ and $\hat{g} \in \invarianceSet$ denoting the element that satisfies $\hat{g}(x) = \orbitMapping(\invarianceSet \cdot x)$ is equivariant. 
\end{proposition}
\begin{proof}
We want to show that a network satisfying the 
condition \eqref{eq:equivariantArgmaxNets} is equivariant. Consider an input $a = r(x)$ to the network, where $r$ denotes an arbitrary element of $S$. We first need to determine the element $\tilde{g} \in S$ such that $\tilde{g}(a) = h(S\cdot a)$. From the definition of the orbit, it follows that $S \cdot x = S \cdot r(x)$, such that our orbit mapping satisfies remains the same, i.e., $h(S \cdot x) = h(S \cdot a) = \hat{g}(x)$. Solving the equation  $\tilde{g}(a) = \hat{g}(x)$ with $a = r(x)$, i.e., $x = r^{-1}(a)$ for $\tilde{g}$ yields $\tilde{g} = \hat{g} r^{-1}$. Now it follows that
\begin{align*}
\net(r(x);\theta) = \net(a;\theta) &= \tilde{g}^{-1}(\netThree( \tilde{g}(a);\theta))\\
& = r(\hat{g}^{-1}(\netThree( \tilde{g}(a);\theta)))\\
& = r(\hat{g}^{-1}(\netThree(\hat{g}(x);\theta)))\\
& = r(\net(x;\theta)),
\end{align*}
which concludes the proof.
\end{proof}

\section{A Discussion on Isometry Invariance}

Here, we will elaborate on how the functional map framework [29] can be seen as an application of our orbit mapping for isometry invariance.
Functional maps are a widely used method to find correspondences between isometric shapes, and we will show here that the framework fits within our proposed theory. 
Non-rigid correspondence is a notoriously hard problem, and joint optimization within a larger framework makes it even more complex.
To resolve this the idea of functional maps is to change the representation of the correspondence from point-wise to function-wise. 
By choosing the eigenfunctions of the Laplace-Beltrami operator [31]  as the basis for functions on the shapes, the problem becomes a least squares problem aligning suitable descriptor functions in the space of functions. 

Here, $F \in \mathcal{F}(\mathcal{X})$ and $G \in \mathcal{F}(\mathcal{Y})$ are descriptor functions on the shapes $\mathcal{X}$ and $\mathcal{Y}$ respectively. They are assumed to take similar values on corresponding points on $\mathcal{X}, \mathcal{Y}$, and generate the designated orbit element within our framework. 
These descriptors are projected onto the eigenfunctions of $\mathcal{X}, \mathcal{Y}$, named $\Phi, \Psi$ respectively. 
These projections are the chosen elements of the orbit we will align, and, for isometries and sufficiently comparable descriptors, the projections can be aligned by an orthogonal transformation generating the group action which is exactly the functional map $C$.
The vanilla functional map optimization looks like this:

\begin{align}
    \underset{C \in O(k)}{\arg\min} \Vert C \Phi^{-1} F - \Psi^{-1} G \Vert_2^2
\end{align}

Functional maps are often used when shape correspondence is required within another framework, and has been used in many deep learning applications [7],[16],[22]. 
Due to its wide application, we will not provide extra experiments to show its efficacy but want to emphasize that this is a possible implementation of our theory. 
\section{Stability of gradient based orbit mapping}
In this section we analyze the stability of the proposed gradient based orbit mapping strategy for discrete images. While the proposed gradient based orbit mapping our approach leads to unique orientation as long as $\int_Z \nabla u(z) ~dz$ is non-zero, practically, the magnitude of $\int_Z \nabla u(z) ~dz$ and interpolation artifacts affect the stability of the orbit mapping. While one could possibly use  forward or central differences to calculate gradients at pixels along approximate circles, this further deteriorates the stability of orbit mapping. This is seen in  Tab.~\ref{tab:stability} a) which shows the mean standard  deviation orientation of orbit-mapped images when input images rotated in steps  of 1 degree using bilinear interpolation. We find that using forward differences to approximate the gradient has the most instability. In the following section, we derive a necessary condition for provable invariance using general convolution kernels (instead of gradients in \textit{x} and \textit{y} direction), where we show that forward differences does not satisfy these conditions for any rotation.

\begin{table}[htb]
    \centering
    \begin{minipage}{0.45\textwidth}
        \centering
        \begin{tabular}{c}
                      a)~\hspace{-1pt}\resizebox{0.99\linewidth}{!}{\begin{tabular}{cccc}
        Dataset&Exact&Central Diff.&Forward Diff\\
        \hline
        CIFAR10 & 10.46&12.47&23.89\\
        CUB200 &9.05&14.56&24.75\\
        \hline
    \end{tabular}}\\
             \\
    c)~\hspace{-1pt}\resizebox{0.99\linewidth}{!}{\begin{tabular}{ccccc}
     Dataset&clean&$\sigma^2$=0.01&$\sigma^2$=0.05&$\sigma^2$=0.1\\
    \hline
    CIFAR10 & 10.46&11.36&14.08&16.69\\
    CUB200 &9.05&10.55&15.99&20.610\\
    \hline
    \end{tabular}}
        \end{tabular}
\end{minipage}%
    \begin{minipage}{0.55\textwidth}
        \centering
       \hspace{2pt} b)~~~~~~~~~~~~~~~~~~~~~~~~~~~~~~~~~~~~~~~~~~\\
       ~~~\includegraphics[width=0.95\linewidth]{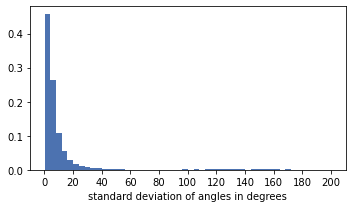}
    \end{minipage}
        \caption{Stability and  robustness of proposed gradient based Orbit Mapping strategy. a)~The mean standard deviation values of angles in degrees over the images in dataset are reported when rotating images based on exact gradients computed along circle using bilinear interpolation, and approximate gradients using finite differences along pixels closest to the circle. b)~The histogram of standard deviations of the predicted orientation in degrees for CIFAR10. c)~The mean standard deviation values of angles in degrees over the images in CIFAR10 dataset are reported, for different levels of additive Gaussian noise.}
    \label{tab:stability}
\end{table}

Tab.~\ref{tab:stability} b) shows the histogram of standard deviations in orientation for CIFAR10 images when calculating exact gradients along the circle. The standard deviations of predicted orientations of over 78\% of the images is less than 10 degrees, and over 44\% of images is less than 4 degrees, indicating a relatively stable orbit mapping for these images. However, a fraction of images also have a higher variance, in predicted orientation possibly due to small values of the integral. Tab.~\ref{tab:stability} c) shows that our gradient based orbit mapping is fairly robust to small additive Gaussian noise.
\section{Invariance to image rotations using convolution kernels}
Following the notation from the paper, let $u(z)$ denote the continuous image function with $z \in \R^2$ representing the spatial coordinates of an image. The invariance set for the orbit of continuous image rotations is 
\begin{align*}
    \begin{split}
    & \invarianceSet = \{g:\inputSpace \rightarrow \inputSpace ~|~ g(u)(z) = (u\circ r(\alpha))(z),
    \text{ for } \alpha \in \R \},\\
   & \text{and }  r(\alpha) = \begin{pmatrix}\cos(\alpha) & -\sin(\alpha)\\ \sin(\alpha) & \cos(\alpha) \end{pmatrix} \text{is the rotation matrix}.
    \end{split}
\end{align*}
Let us consider two kernels $k_i:\R^2 \rightarrow \R$, $i=\{1,2\}$. We now investigate the convolution of a kernel with a rotated image $ (u\circ r(\alpha))(z)$
\begin{align*}
\begin{split}
   \left(k_i * u\circ r(\alpha)\right)(z) = \int_{\R^2} k_i(x)(u\circ r(\alpha))(z-x)dx\\
   =\int_{\R^2} k_i(x)u(r(\alpha)z-r(\alpha)x)dx\\
   = \int_{\R^2} k_i(r^T \varphi)u(r(\alpha)z-\varphi)d\varphi \\\text{ with } \varphi=r(\alpha)x
\end{split}
\end{align*}
Now assume 
\begin{align}
    \label{eq:condition_kernels}
    \begin{split}
   &\begin{pmatrix} k_1(r^T(\alpha) \varphi) \\k_2(r^T(\alpha) \varphi)\end{pmatrix}= r^T(\alpha)\begin{pmatrix}k_1( \varphi) \\k_2(\varphi)\end{pmatrix}.
    \end{split}
\end{align}
Then
\begin{align*}
    \begin{split}
   \begin{pmatrix} \left(k_1 *( u\circ r(\alpha))\right)(z) \\\left(k_2 * (u\circ r(\alpha))\right)(z)\end{pmatrix}=\int_{\R^2} r^T(\alpha)\begin{pmatrix}k_1( \varphi) \\k_2(\varphi)\end{pmatrix}u(r(\alpha)z-\varphi)d\varphi.\\
  =r^T(\alpha)\begin{pmatrix}(k_1*u)(r(\alpha) z) \\(k_2*u)(r(\alpha) z)\end{pmatrix}\\
    \end{split}
\end{align*}
Then for a suitable set $Z$ which makes the integral rotationally invariant, (e.g. circles around image center)
\begin{align}
 \begin{split}
 \int_Z \begin{pmatrix} \left(k_1 *( u\circ r(\alpha))\right)(z) \\\left(k_2 *( u\circ r(\alpha))\right)(z)\end{pmatrix} dz 
  =  r^T(\alpha)\int_{Z} \begin{pmatrix}(k_1*u)(\varphi) \\(k_2*u)(\varphi)\end{pmatrix}d\varphi
 \end{split}
\end{align}
And we can determine the optimal rotation as solution to
\begin{align}
\hat{g} = \argmax_{g \in \invarianceSet}& \left\langle \begin{pmatrix}1\\0 \end{pmatrix}, \int_Z \begin{pmatrix}
  k_1*u\\k_2*u
\end{pmatrix}(z) ~dz  \right\rangle
\end{align} 
whose solution is given by  $\hat{\alpha}$ such that
\begin{align}
 \begin{pmatrix}
   \cos\hat{\alpha} \\\sin\hat{\alpha}
 \end{pmatrix}
 =  \frac{\int_Z \begin{pmatrix}
  k_1*u\\k_2*u
\end{pmatrix}(z) ~dz}{\left\lVert\int_Z \begin{pmatrix}
  k_1*u\\k_2*u
\end{pmatrix}(z) ~dz\right\rVert} 
\label{eq:kernelRotation}
\end{align}
We can see that \eqref{eq:condition_kernels} is a necessary condition to ensure invariance to image rotations using orbit mapping with \eqref{eq:kernelRotation} employing convolution kernels $k_1$ and $k_2$. For discrete convolution kernels, eq.~\eqref{eq:condition_kernels} is not exactly satisfied for arbitrary rotations due to discretization problem. We can deduce necessary conditions on discrete kernels $k_1$ and $k_2$ to satisfy eq.~\eqref{eq:condition_kernels} for rotations in multiples of $90^o$. For square kernels $k_1$ and $k_2$ of size $N\times N$, we find that 
\begin{align}
\label{eq:k1}
   k_1[i,j]= k_1[N-i+1,N-j+1] \text{ and}\\
   k_2 = k_1\circ r(-90^o)
   \label{eq:k2}
\end{align}
are necessary to satisfy the condition \eqref{eq:condition_kernels} for $\alpha =90^o$.\\
For $N=2$, this gives kernels of the form
\begin{align*}
    k_1 =\begin{pmatrix}
      a&b\\-b &-a
    \end{pmatrix} \text{ and }
    k_2 =\begin{pmatrix}
      -b&a\\-a& b
    \end{pmatrix} 
\end{align*}
For $N=3$,
\begin{align*}
    k_1 =\begin{pmatrix}
      a&b&c\\d&0&-d\\-c&-b &-a
    \end{pmatrix} \text{ and }
    k_2 =\begin{pmatrix}
      -c&d&a\\-b&0&b\\-a&-d&c
    \end{pmatrix} 
\end{align*}
Note that computing gradients using central differences satisfies \eqref{eq:k1} and \eqref{eq:k2}, whereas using forward differences does not satisfy these conditions. Therefore, we observe more instabilities in orbit mapping when forward differences are used for gradient computation, see Tab.~\ref{tab:stability}.
\section{Details about the Experimental Setting}
In the following we provide the detailed training settings used in our experiments.
\subsection{Rotation invariance for images}
For our experiments with image rotational invariance, we used  Pytorch(v.1.8.1), python(v.3.8.8), torchvision(v.0.9.1). The exact training protocol is provided below.\\
\textbf{CIFAR10} We trained a Resnet18~\cite{he2016deep} on the CIFAR 10 dataset, using stochastic gradient descent with initial learning rate 0.1, momentum 0.9, and weight decay 5e-4. Additionally, we trained a small Convnet and a linear model which used an initial learning rate of 0.01.  For all the models, the learning rate is decayed by a factor of 0.5 whenever the validation loss does not decrease for 5 epochs.  Training data is augmented using random horizontal flips, random crops of size 32 after zero-padding by 4 pixels. We divide the training data into train (80\%) and validation (20\%) sets. Networks are trained for 150 epochs with a batch size of 128 and we report the results on the test set using the model with best validation accuracy. The experiments with CIFAR10 were performed partially on a machine with one Nvidia TITAN RTX, and partially on machine with 4 NVIDIA GeForce RTX 2080  GPUs.\\
\textbf{HAM10000} We fine-tuned an imagenet pretrained\footnote{ pretrained model from \url{https://github.com/rwightman/pytorch-image-models} licensed Apache 2.0} NFNet-F0~\cite{brock2021high} on HAM10000 dataset \cite{tschandl2018ham10000}. The dataset is split into  8912 train and 1103 validation images using stratified split, ensuring there are no duplicates with the same lesion ids in the train and validation sets. Training data is augmented using random horizontal and vertical flips and color jitter, and randomly oversample the minority classes to mitigate class imbalance.  The network is finetuned for 5 epochs, with a batch size of 128 and learning rate  of 1e-4, weight decay of 5e-4 using Adam optimizer~\cite{kingma2014adam}  with exponential learning rate decay, with factor 0.2.  For training using TI-pool which uses 4 rotated copies of images, we reduce the batch size to 32 to fit the GPU memory. For experiments with STN we use a 3 layered CNN with convolution filers of size $3\times3$ followed by 2 fully connected layers for pose prediction. For experiment with ETN we use a CNN with 4 conv layers with 64 channels and 2 fully connected layers for pose prediction. We report results using final iterate on the validation set. The experiments with HAM10000 dataset were partially performed on a machine with one NVIDIA TITAN RTX card, and partially on machine with 4 NVIDIA GeForce RTX 2080  GPUs.\\
\textbf{CUB200}
This is a small  dataset  containing 11,788 images of birds, split into 5994 images for training and 5794 test images. Since training a network from scratch gives low accuracies (around 35\% clean accuracy with Resnet-50), we instead perform finetuning using an imagenet pretrained Resnet-50 from pytorch torchvision~(v.0.9.1)  on CUB-200 dataset ~\cite{wah2011caltech}. The training data is augmented using random horizontal flips, random resized crops of size 224.  The network is finetuned for 60 epochs with batch size of 128 and initial learning rate of 1e-4, using Adam optimizer~\cite{kingma2014adam} , weight decay of 5e-4, with exponential learning rate decay, with factor 0.9. . For training using TI-pool which uses4 rotated copies of images, we reduce the batch size to 64 to fit in the GPU memory. For experiment with ETN we use a CNN with 4 conv layers with 64 channels and 2 fully connected layers for pose prediction.  We report the accuracies using the final iterate on the test set. The experiments on CUB-200 dataset were performed on machine with 4 NVIDIA GeForce RTX 2080  GPUs. 

All the  three image  datasets including  HAM10000 dataset \cite{tschandl2018ham10000}  used in our experiments are publicly available and widely used in machine learning literature. To the best of our knowledge these do not contain offensive content or personally identifiable information. 

\subsection{Rotation and Scale invariance for 3D point clouds}
We investigate invariance to rotations and scale for 3D point clouds with the task of point cloud classification on the \textit{modelnet40} dataset \cite{wu20153d}. For this dataset note the asset descriptions at \url{https://modelnet.cs.princeton.edu/}: "All CAD models are downloaded from the Internet and the original authors hold the copyright of the CAD models. The label of the data was obtained by us via Amazon Mechanical Turk service and it is provided freely. This dataset is provided for the convenience of academic research only." We use the resampled version of \url{shapenet.cs.stanford.edu/media/modelnet40_normal_resampled.zip}. We follow the hyperparameters of \cite{qi2017pointnet,qi2017pointnetpp} with improvements from the implementation of \cite{yan_pointnet} on which we base our experiments. We train a standard PointNet for 200 epochs with a batch size of 24 with Adam \cite{kingma2014adam} with base learning rate of $0.001$, weight decay of $0.0001$. During training we sample 1024 3D points from every example in \textit{modelnet40}, randomly scale with a scale from the interval $[0.8, 1.25]$, and randomly translate by an offset of up to $0.1$ - if not otherwise mentioned in our experiments. This is the training procedure proposed in \cite{yan_pointnet}. However, we always train the model for the the full 200 epochs and report final \textit{class} accuracy based on the final result - we do not report instance accuracy. We further report invariance tests based on the final model. 

As described in the main body, we evaluate rotational invariance by testing on $16\times 16$ regularly spaced angles from $[0, 2\pi]$, rotating along $xy$ and $yz$ axes. We evaluate scaling invariance by testing the scales $\lbrace 0.001, 0.01, 0.1, 0.5, 1.0, 5.0, 10, 100, 1000\rbrace$.
All experiments for this dataset were run on three single GPU office machines, containing an NVIDIA TITAN Xp, and two GTX 2080ti, respectively.

\section{Additional Numerical Results}
\subsection{Invariance to continous image rotations}
\textbf{Discretization effects in CUB200 }
We further investigate the effect of  discretization using different interpolation schemes for rotation on higher resolution on the CUB-200 dataset (trained at 224x224 resolution) fine-tuned using Resnet-50. Tab.~\ref{tab:rob.train_reg_cub.} shows the results of different training schemes with and without our orbit mapping (\textit{OM}) obtained when using different interpolation schemes for rotation. Besides standard training (\textit{Std.}), we use rotation augmentation (\textit{RA}), and the adversarial training and regularization from \cite{engstrom2019exploring,yang2019invariance}. Even for this higher resolution dataset, the worst-case accuracies between different types of interpolation may differ by more than $15\%$.
\begin{table*}[htb]
\centering
\resizebox{\linewidth}{!}{
\begin{tabular}{l l l l l l  c l l l }
    \hline
 Train& OM &Clean.&  \multicolumn{3}{c}{Average}& \multicolumn{3}{c}{Worst-case}\\
\cline{4-6}\cline{8-10}
&&&Nearest&Bilinear&Bicubic&&Nearest&Bilinear&Bicubic\\
\hline
 \multirow{2}{*}{Std.}  &\ding{55}&\textbf{77.41$\pm$0.33}&37.67$\pm$0.35& 52.45$\pm$0.29 &51.87$\pm$0.31&&3.19$\pm$0.49 &8.07$\pm$0.35& 8.16$\pm$0.33
\\
&\ding{51}Train+Test&71.19$\pm$0.34&63.35$\pm$0.30& 71.56$\pm$0.34 &70.93$\pm$0.35&&40.63$\pm$0.48 &58.80$\pm$0.39& 59.02$\pm$0.41\\
\hline
 \multirow{3}{*}{RA.}  &\ding{55}&69.89$\pm$0.28&67.61$\pm$0.33& 70.12$\pm$0.34 &68.83$\pm$0.37&&34.88$\pm$0.47 &41.01$\pm$0.41& 40.50$\pm$0.43
\\
&\ding{51}Test&69.41$\pm$0.31&69.19$\pm$0.32& 69.27$\pm$0.29 &68.53$\pm$0.38&&48.63$\pm$0.43 &56.28$\pm$0.39& 55.86$\pm$0.40\\
&\ding{51}Train+Test&70.35$\pm$0.46&69.41$\pm$0.23& 70.72$\pm$0.18 &70.37$\pm$0.34&&47.92$\pm$0.26 &57.54$\pm$0.39& 57.62$\pm$0.14 \\
\hline
Advers.&\ding{55}& 64.54$\pm$0.17&53.74$\pm$0.65& 64.07$\pm$0.25 &63.22$\pm$0.54&&26.63$\pm$0.79 &42.82$\pm$0.60& 42.44$\pm$0.55
\\
 Mixed &\ding{55}&
68.56$\pm$0.46&57.17$\pm$0.60& 65.91$\pm$0.42 &65.76$\pm$0.51&&28.06$\pm$0.58 &42.87$\pm$0.32& 42.92$\pm$0.38
\\
Advers.-KL&\ding{55}&
64.47$\pm$0.35&53.93$\pm$0.35& 64.65$\pm$0.26 &64.02$\pm$0.34&&26.94$\pm$0.46 &43.04$\pm$0.63& 42.61$\pm$0.37
\\
Advers.-ALP&\ding{55}&
64.63$\pm$0.31&55.56$\pm$0.67& 64.34$\pm$0.17 &63.21$\pm$0.24&&29.55$\pm$0.69 &43.63$\pm$0.21& 43.48$\pm$0.32
\\
ETN & \ding{55}&64.14$\pm$0.24&64.26$\pm$0.65& 66.95$\pm$0.42 &64.32$\pm$0.62&&43.33$\pm$1.01 &52.85$\pm$1.12& 49.72$\pm$1.31
\\
TIpool&\ding{55}&76.80$\pm$0.25&60.67$\pm$0.79& 74.90$\pm$0.15 &74.82$\pm$0.24&&36.06$\pm$1.12 &59.04$\pm$0.37& 59.50$\pm$0.41
\\
TIpool-RA&\ding{55}&73.47$\pm$0.48&72.30$\pm$0.51& 74.71$\pm$0.29 &73.65$\pm$0.36&&57.22$\pm$0.64 &62.82$\pm$0.56& 62.31$\pm$0.42
\\
TIpool&\ding{51}Train+Test&76.82$\pm$0.15&68.50$\pm$0.58& \textbf{77.18$\pm$0.18} &\textbf{77.04$\pm$0.16}&&49.85$\pm$0.65 &\textbf{69.19$\pm$0.36}& \textbf{69.64$\pm$0.33}\\
TIpool-RA&\ding{51}Train+Test&74.78$\pm$0.20&\textbf{73.79$\pm$0.48}& 75.89$\pm$0.17 &75.07$\pm$0.16&&\textbf{59.57$\pm$0.57} &67.78$\pm$0.20& 67.64$\pm$0.18\\

\hline
\end{tabular}
}
\vspace{0.5pt}
\caption{Effect of augmentation and including gradient based orbit mapping \textit{(OM)} on robustness to rotations with different interpolations for CUB200 classification using Resnet50. Shown are clean accuracy on standard test set and  average and worst-case accuracies on rotated test set. Mean and standard deviations over 5 runs are reported. }
\label{tab:rob.train_reg_cub.}
\end{table*}

\begin{table*}[htb]
\centering
\resizebox{\linewidth}{!}{
\begin{tabular}{l |l l l    l l l  c l l l }
\hline
Network &  Train& OM& Std.&  \multicolumn{3}{c}{Average}& \multicolumn{3}{c}{Worst-case}\\
\cline{5-7}\cline{9-11}
&&&&Nearest&Bilinear&Bicubic&&Nearest&Bilinear&Bicubic\\
\hline
\multirow{6}{*}{Linear}    & \multirow{2}{*}{Std.}      &\ding{55}&  \textbf{38.89$\pm$0.17}&25.31$\pm$0.21& 25.57$\pm$0.22 &25.48$\pm$0.24&&2.50$\pm$0.11 &3.56$\pm$0.17& 3.26$\pm$0.11\\
&     &\ding{51}Train+Test& 31.87$\pm$0.10&\textbf{31.25$\pm$0.04}& \textbf{31.58$\pm$0.05} &\textbf{31.33$\pm$0.04}&&13.08$\pm$0.23 &18.85$\pm$0.21& 18.21$\pm$0.21 \\
\cline{2-11}

&\multirow{3}{*}{RA}  &\ding{55}& 29.73$\pm$0.18&30.66$\pm$0.03& 30.77$\pm$0.03 &30.72$\pm$0.03&&14.30$\pm$0.42 &18.31$\pm$0.29& 16.94$\pm$0.37\\ 
&    & \ding{51}Test  &	30.60$\pm$0.13&30.52$\pm$0.07& 30.65$\pm$0.08 &30.54$\pm$0.09&&16.83$\pm$0.47 &21.17$\pm$0.28& 20.37$\pm$0.26\\
&    & \ding{51}Train+Test  &31.06$\pm$0.26&31.07$\pm$0.11& 31.27$\pm$0.10 &31.13$\pm$0.09&&\textbf{19.19$\pm$0.28} &\textbf{24.25$\pm$0.31}&\textbf{23.68$\pm$0.31}\\

\cline{2-11}
&  Advers.         &\ding{55}&28.82$\pm$0.77&29.46$\pm$0.60& 29.62$\pm$0.56 &29.36$\pm$0.56&&11.45$\pm$0.81 &14.20$\pm$0.93& 13.65$\pm$0.55\\ 
\hline
\hline
\multirow{6}{*}{Convnet} &\multirow{2}{*}{Std.}     &\ding{55} &
\textbf{86.12$\pm$0.33}&32.01$\pm$0.32& 35.97$\pm$0.26 &38.15$\pm$0.36&&0.85$\pm$0.09 &0.57$\pm$0.06& 0.89$\pm$0.14
\\
&&\ding{51}Train+Test &76.13$\pm$0.96&64.34$\pm$0.35& 71.21$\pm$0.96 &\textbf{74.61$\pm$0.84}&&25.78$\pm$0.49 &49.60$\pm$0.79& 55.57$\pm$0.81\\

\cline{2-11}   

                          & \multirow{3}{*}{RA}    &\ding{55} &75.03$\pm$0.99&71.77$\pm$0.84& 65.45$\pm$0.66 &70.22$\pm$0.66&&27.96$\pm$0.50 &27.06$\pm$0.61& 32.51$\pm$0.53\\ 
                        &    & \ding{51}Test &70.12$\pm$0.64&67.64$\pm$0.55& 61.03$\pm$0.67 &66.09$\pm$0.71&&39.01$\pm$0.57 &42.88$\pm$0.90& 49.39$\pm$0.68\\
                          &    & \ding{51}Train+Test &74.30$\pm$0.77&\textbf{73.24$\pm$0.58}& 69.52$\pm$0.53 &73.38$\pm$0.59&&\textbf{46.25$\pm$0.54} &\textbf{53.36$\pm$0.57}& \textbf{59.04$\pm$0.53}\\

\cline{2-11}   
& Advers.         &\ding{55}  &72.96$\pm$0.95&62.08$\pm$0.59& \textbf{74.29$\pm$0.88} &73.86$\pm$0.76&&26.24$\pm$0.43 &50.99$\pm$0.54& 52.46$\pm$0.51\\

\hline
\hline
\multirow{6}{*}{Resnet18} &
\multirow{2}{*}{Std.}  &  \ding{55} & \textbf{93.98$\pm$0.32} & 35.12$\pm$0.81& 40.06$\pm$0.44 &42.81$\pm$0.50&&0.79$\pm$0.38 &1.31$\pm$0.13& 2.22$\pm$0.17\\
&&\ding{51} Train+Test& 87.99$\pm$0.43&72.40$\pm$0.33& \textbf{84.12$\pm$0.55} &\textbf{86.61$\pm$0.49}&&34.57$\pm$0.94 &68.60$\pm$0.81& 74.49$\pm$0.84\\

\cline{2-11} 
&\multirow{3}{*}{RA}& \ding{55} &
85.54$\pm$0.72&80.47$\pm$0.74& 75.99$\pm$0.72 &79.47$\pm$0.65&&45.50$\pm$0.83 &44.71$\pm$0.74& 50.50$\pm$0.78\\
& & \ding{51} Test    &79.26$\pm$0.42
 &	74.93$\pm$0.51& 69.31$\pm$0.65 &73.94$\pm$0.63&&48.93$\pm$0.75 &52.18$\pm$0.91& 58.69$\pm$0.78
\\
&&  \ding{51} Train+Test& 85.40$\pm$0.57&\textbf{84.37$\pm$0.58}& 81.82$\pm$0.59 &84.82$\pm$0.52&&\textbf{66.22$\pm$0.75} &\textbf{71.09$\pm$1.01}&\textbf{76.44$\pm$0.89}\\
\cline{2-11}

&Advers.& \ding{55}&69.32$\pm$1.61&61.73$\pm$1.12& 68.54$\pm$0.68 &68.00$\pm$0.31&&36.95$\pm$0.97 &50.21$\pm$0.55& 49.73$\pm$0.98
\\
\hline
\end{tabular}
}
\vspace{0.5pt}
\caption{Comparing  rotational  invariance  using  training  schemes  vs. orbit mapping for CIFAR10 classification using \textit{i)~Linear network ii)~5-layer Convnet iii)~Resnet18}. Shown are the mean clean accuracy and the average and worst case accuracies when test images are rotated in steps of 1 degree. The mean and standard deviation values over 5 runs are reported.}
\label{tab:cifar10_supple}
\end{table*}
In particular, adversarial training with bi-linear interpolation is still more vulnerable to image rotations with nearest neighbor interpolation. Even the learned ETN also exhibits similar behavior. While our approach is also affected by the interpolation effects, the vulnerability to nearest neighbor interpolation is ameliorated when using rotation augmentation. We obtain best results using orbit mapping  in conjunction with the discrete invariant approach \cite{laptev2016ti}\vspace{1em} \\
\textbf{Effect of Network architecture for CIFAR10} To investigate the effectiveness of our approach, we experiment three different network architectures:\textit{~i)~a linear network, ii)~a 5-layer convnet ii)~a Resnet18.} We compare the performance of our orbit mapping approach with training schemes, i.e. augmentation and adversarial training  for rotational invariance in Tab.~\ref{tab:cifar10_supple}. For all the three architectures considered,  our orbit mapping  together with rotation augmentation consistently results in  the most accurate predictions in the worst case.\vspace{1em} \\
\textbf{Comparing Computation Complexity for CIFAR10} In Tab.~\ref{tab:cifar_train_time}, the training times using different approaches are compared for rotation-invariant CIFAR10 classification. It can be noted that the proposed gradient based orbit mapping is significantly easier and computationally cheaper to train in comparison with other approaches for incorporating invariance. In contrast, adversarial training is the most computationally expensive approach.\\ 
\begin{table*}[htb]
\small
\centering
\resizebox{0.8\linewidth}{!}{
\begin{tabular}{ c | c | c | c | c | c }
Method & Std. & STN & ETN & Adv. & OM\\
\hline
Train-time/epoch&18.05$\pm$0.05&18.90$\pm$0.05&18.89$\pm$0.07&72.09$\pm$0.18&18.59$\pm$0.04\\
\end{tabular}}
\caption{Average training time per epoch in seconds for different approaches to incorporate rotation invariance, with Resnet18 as base architecture for CIFAR10 classification. Training time correspond to  runs on a machine with single Titan-RTX GPU.\label{tab:cifar_train_time}}
\end{table*}
\\
\textbf{Comparing Computational Complexity of ROTMNIST}
\begin{table*}[h]
\small
\begin{center}
\begin{tabular}{ c | c |  c }
 \multirow{2}{*}{OM} &  \multicolumn{1}{c|}{D4/C4}& \multicolumn{1}{c}{D16/C16}\\
 & Train-time/epoch &Train-time/epoch \\
\hline
\ding{55} &  4.47 s & 41.89 s
\\
 \ding{51} & 4.78 s & 42.08 s
 \\

\end{tabular}
\end{center}
\caption{Comparing computational complexity of D4/C4 and D16/C16 models. Orbit mapping adds no learnable parameters and increases training time very marginally ($\sim$0.3 seconds/epoch). Training times correspond to  runs on a machine with single Titan-RTX GPU.\label{tab:rot_mnit_train_time}}
\end{table*}
Tab.~\ref{tab:rot_mnit_train_time} compares the computational complexity of the D4/C4 and D16/C16 models. The D16/C16 model has significantly higher computational complexity than the D4/C4 model,  though the number of learnable parameters is nearly same. The network size of D16/C16 network s higher due to more rotated copies of the filters, resulting in larger training and inference times. Orbit mapping adds no learnable parameters and increases training time very marginally ($\sim$0.3 seconds/epoch). Training times correspond to  runs on a machine with single Titan-RTX GPU.

\clearpage
%
%

\bibliographystyle{splncs}

\end{document}